\documentclass[11pt]{article}
\usepackage{fullpage}

\usepackage{authblk}

\date{}

\title{\bfseries\papertitle}

\author[1]{Aldo Pacchiano$^*$}
\author[2]{Aadirupa Saha$^{*\#}$}
\author[3]{Jonathan Lee}
\affil[1]{Microsoft Research, New York City}
\affil[2]{TTI Chicago}
\affil[3]{Stanford University}

\usepackage[round]{natbib}
\usepackage[T1]{fontenc}
\usepackage{lmodern}
\usepackage[utf8]{inputenc} 
\usepackage{nameref}
\usepackage[colorlinks, linkcolor=blue, filecolor = blue, citecolor = blue, urlcolor = blue]{hyperref}
\usepackage{url}            
\usepackage{booktabs}       
\usepackage{amsfonts}       
\usepackage{nicefrac}       
\usepackage{microtype}      
\usepackage{xcolor}         

\usepackage{graphics}
\usepackage{amssymb}
\usepackage{latexsym}
\usepackage{amsthm}
\usepackage{amsmath}
\usepackage{xcolor}
\usepackage{mathtools}
\usepackage{enumitem}
\usepackage{thm-restate}
\usepackage{algorithm}
\usepackage{algorithmic}
\usepackage{graphicx}
\usepackage[normalem]{ulem}
\usepackage{caption}
\usepackage{subcaption}

\newtheorem{thm}{Theorem}

\newtheorem{rem}[thm]{Remark}

\newcommand{\R}{{\mathbb R}}

\newcommand{\N}{{\mathbb N}}

\newcommand{\E}{{\mathbf E}}
\newcommand{\I}{{\mathbf I}}
\newcommand{\1}{{\mathbf 1}}

\newcommand{\cA}{{\mathcal A}}

\newcommand{\cB}{{\mathcal B}}

\newcommand{\cS}{{\mathcal S}}

\newcommand{\V}{{\mathbf V}}

\renewcommand{\u}{{\mathbf u}}
\renewcommand{\v}{{\mathbf v}}
\newcommand{\w}{{\mathbf w}}
\newcommand{\x}{{\mathbf x}}

\newcommand{\mrl}{Preference based Linear-RL}

\newcommand{\algrrl}{LPbRL}

\newcommand{\papertitle}{Dueling RL: Reinforcement Learning with\\ Trajectory Preferences}

\usepackage{amsmath,amsfonts,amsthm,bm}
\usepackage{bbm}

\newcommand{\comment}[1]{}

\newtheorem{lemma}{Lemma}

\newtheorem{assumption}{Assumption}

\renewcommand{\SS}{\Pi}
\newcommand{\CC}{\mathcal{C}}

\newcommand{\eq}[1]{\begin{align}#1\end{align}}
 
\def\eqref#1{equation~\ref{#1}}

\def\ceil#1{\left\lceil #1 \right\rceil}

\def\1{\mathbbm{1}}

\def\rmV{{\mathbf{V}}}

\DeclareMathAlphabet{\mathsfit}{\encodingdefault}{\sfdefault}{m}{sl}
\SetMathAlphabet{\mathsfit}{bold}{\encodingdefault}{\sfdefault}{bx}{n}

\DeclareMathOperator*{\argmax}{arg\,max}
\DeclareMathOperator*{\argmin}{arg\,min}

\renewcommand{\eqref}[1]{Eq. (\ref{#1})}

\newtheorem{definition}{Definition}

\theoremstyle{definition}

\newcommand{\aldo}[1]{\textcolor{red}{\it {\bf Aldo:} {#1}}}

\newcommand{\lee}[1]{\textcolor{blue}{\it {\bf Jon:} {#1}}}

\makeatletter
\newcommand{\pushright}[1]{\ifmeasuring@#1\else\omit\hfill$\displaystyle#1$\fi\ignorespaces}
\newcommand{\pushleft}[1]{\ifmeasuring@#1\else\omit$\displaystyle#1$\hfill\fi\ignorespaces}
\makeatother

\begin{document}

\maketitle

\def\thefootnote{*}\footnotetext{Equal contribution alphabetically.}
\def\thefootnote{\#}\footnotetext{Author is currently with Apple ML Research. Majority of the work was done when the author was at MSR, NYC and TTI, Chicago.}

\begin{abstract}
We consider the problem of preference-based reinforcement learning (PbRL), where, unlike traditional reinforcement learning (RL), an agent receives feedback only in terms of 1 bit (0/1) preferences over a trajectory pair instead of absolute rewards for it.
%
%
The success of the traditional reward-based RL framework crucially depends on how accurately a system designer can express an appropriate reward function, which is often a non-trivial task.
The main novelty of the our framework is the ability to learn from preference-based trajectory feedback that eliminates the need to  hand-craft numeric reward models.  
%
This paper sets up a formal framework for the PbRL problem with non-Markovian rewards, where the trajectory preferences are encoded by a generalized linear model of dimension $d$. 
Assuming the transition model is known, we propose an algorithm with a regret guarantee of $\tilde {\mathcal{O}}\left( SH d  \log (T / \delta) \sqrt{T}  \right)$.
 We further extend the above algorithm to the case of unknown transition dynamics and provide an algorithm with  regret $\widetilde{\mathcal{O}}((\sqrt{d} + H^2 + |\mathcal{S}|)\sqrt{dT} +\sqrt{|\mathcal{S}||\mathcal{A}|TH} )$.
To the best of our knowledge, our work is one of the first to give tight regret guarantees for preference-based RL problem with trajectory preferences. 
\end{abstract}



\vspace{-10pt}
\section{Introduction}
\label{sec:intro}

Classical reinforcement learning (RL) with absolute reward feedback is a well-studied framework which is a sequential experience-driven learning process to optimize an accumulated long-term reward %
 \citep{sutton18,ucrl,singh02}. 
Over the years, several works have addressed RL in terms of both the optimal sample complexity for finding the best policy \citep{az13,dan15,dan17,valko20,torpac} and minimizing regret via balancing exploration and exploitation  \citep{zhang19,az17,ortner20,talebi18,efroni20,domingues20}.

However, a major limitation of the standard RL setting is that its success crucially depends on the prior knowledge encoded into the definition of the reward
function. The learned policy can often be sensitive to small changes of the reward, possibly yielding very different behaviors depending on the relative values of the rewards. The choice of reward function in applications such as robotics consequently entails a high amount of non-trivial effort in reward engineering, leading to challenges such as reward shaping, reward hacking, infinite rewards, and multi-objective outcomes \citep{wirth13,wirth17}.

The framework of \emph{Preference-based Reinforcement Learning} (PbRL) \citep{busa14,wirth16,wirth17} has been proposed as a fix to this problem, to enforce learning from non-numerical, relative feedback which need not suffer from issues due to the inaccuracy of reward modeling or engineering. 
This framework widely applies to multiple areas including robot training, stock-prediction, recommender systems, clinical trials, etc. \citep{sui19,sadigh17,nipsPRL17,kupcsik18,jain13,wirth17}. 

While the problem of PbRL was introduced almost a decade ago, most work in it has been primarily applied or experimental in nature \citep{jain13,busa14,nipsPRL17,wirth13,wirth16,wirth17,kupcsik18}. There have also been attempts to design suitable algorithms based on varying preference models and problem objectives \citep{sui19,xu20}, but, to the best of our knowledge, existing theoretical guarantees on PbRL literature are sparse. The performance guarantees of most of the proposed algorithms are not well-understood \citep{wirth17,xu20} except for some very recent attempts \citep{sui19,xu20} as discussed below in the section on related work. 
We consider the problem of provably finding the best finite-horizon policy (i.e., one with highest expected reward) for an unknown Markov decision process (MDP), but with only relative preference feedback on $H$-length trajectories.

\textbf{Problem Setup (informal). } Consider a $T$-round, $H$-horizon MDP $(\mathbb P, \mathcal S, \mathcal A, H)$, with $\cS$ and $\cA$ being the finite sets of states and actions, respectively, and $\mathbb P$ representing the transition dynamics of the MDP. 
We consider a real-valued score function $s(\tau)$, which is neither known nor queryable, that scores a given trajectory $\tau$. We assume preference of any two trajectories $\tau_1$ and $\tau_2$ is determined by their underlying score difference, i.e. $P(\tau_1 \succ \tau_2) = \sigma\big( s(\tau_1) - s(\tau_2) \big)$, $\sigma: \R \mapsto [0,1]$ being a suitable link-function. In particular, we assume $s(\cdot)$ is an (unknown) linear function of the trajectory-feature $\phi(\tau) \in \R^d$, and the link-function $\sigma$ is the sigmoid \cite{li17}. {The goal is to minimize the regret with respect to the optimal policy}. 

An important thing to note is that in our setting the trajectory features $\phi(\tau) \in \R^d$ are not necessarily  sum-decomposable (over individual state-action features of the trajectory) and the underlying reward function is non-Markovian. In this case, the optimal policy may be history dependent. This is more general than assuming the trajectory reward is a linear function of the sum of per-state features, e.g. in \cite{sui19}. Under the latter more limiting assumption, the traditional linear bandit techniques can be easily used to derive regret guarantees. Since the number of history dependent policies is super exponential, to deal with this more general setting we first show the log-covering number of the history dependent policies that are an optimal policy of an MDP with a trajectory score specified by our form of trajectory feedback is upper bounded by a polynomial quantity. 
%
Our specific contributions are as follows:

\begin{enumerate}
	\item To the best of our knowledge, we are the first to formulate and analyze the finite time regret guarantee for preference-based linear bandits problem with non-Markovian reward models (Sec. \ref{sec:prob}). 
	
	\item We propose an algorithm for known transition dynamics which is shown to yield a regret guarantee of $\tilde {\mathcal{O}}\left( S H d  \log (T / \delta) \sqrt{T}  \right)$ (Sec. \ref{sec:known-model}). \footnote{The notation $\widetilde{\mathcal{O}}(\cdot)$ hides logarithmic factors in $T, H, |\mathcal{A}|, |\mathcal{S}|$.} 
	
	\item We further generalize our algorithm to the case of unknown models and propose an algorithm with regret guarantee $\widetilde{\mathcal{O}}((\sqrt{d} + H^2 + |\mathcal{S}|)\sqrt{dT} +\sqrt{|\mathcal{S}||\mathcal{A}|TH} )$ (Sec. \ref{sec:unknown-model}).
\end{enumerate}

\textbf{Related Work.} 
Over the last two decades the problem of learning from preference feedback in bandits, known as \textit{dueling bandits}, has gained much attention \citep{Yue+12,Zoghi+14RCS,Zoghi+15}. Dueling bandits  generalizes the standard multi-armed bandit (MAB) \citep{Auer+02}. The goal is to identify a set of 'good' arms from a larger fixed set of arms by querying preference feedback for pairs of actively chosen arms.
\cite{Yue+09,BTM,RDB,GS21,SG18}
The setting is relevant in various real-world systems which aim to collect information from user preferences, including recommender systems, retail management, search engine optimization, job scheduling, etc. 
Towards these goals, several algorithms have been proposed \citep{Ailon+14,Zoghi+14RUCB,Komiyama+15,Adv_DB,SGrank18,SGwin18}.

Though there has been a fair amount of research for preference-based bandits (no state information), few works consider incorporating preference feedback in the reinforcement learning (RL) framework, which considers the problem of long-term objectives over a markov decision process \citep{singh02,ng06,talebi18,ortner20, zhang19,zanette19}. However the classical RL setup assumes access to reward feedback for each state-action pair which might be impractical in many real world scenarios.  
Few very recent works considered  training RL agents based on general trajectory-based reward which are available only at the end of each trajectory \citep{Efroni21,chatterji21}, but their setting still assumes access to absolute reward feedback, unlike the case in PbRL. 
Some initial works consider the applied PbRL problem inspired by the problems of reward hacking,  reward shaping,  difficulty to model infinite rewards or  multi-objective trade-offs \citep{busa14,wirth16,wirth17,nipsPRL17} etc. 

\cite{sui19} made the first attempt to analyze the finite $T$-round regret guarantee for the PbRL problem with trajectory preference feedback, where the learner is allowed to run two independent trajectories in parallel and receive $0/1$ preference feedback after every such $h$-length roll out. Assuming an underlying MDP model, the preference between two $h$-length trajectories is modeled as being proportional to the accumulated reward of the corresponding trajectories. The authors propose a Double Posterior Sampling (DPS) technique with asymptotically sublinear regret. 

\cite{xu20} models reward-free trajectory preferences and analyses the sample complexity of finding the $\epsilon$-best-policy. Their proposed algorithm crucially depends on an underlying dueling bandit black box whose performance guarantee is restricted to preference structures like Strong Stochastic Transitivity and Stochastic Triangle Inequality. Furthermore, the algorithms proposed in this work are not shown to enjoy provably optimal sample complexity, and, moreover, the fundamental performance limit of sample complexity is also not explicitly analyzed.

The literature of multi-agent reinforcement learning in Markov games closely relates to the setup of PbRL which attempts the problem of reaching Nash equilibrium of a simultaneous move markov game based on per-state win-loss feedback of the two (or multiple) players. \cite{bai20,bai+20,liu21} address the problems from finite action two player markov games, while \cite{xie+20} extended this setting to zero sum games with linear function approximation. However all these works analyzed the episodic sample complexity of the learning algorithm towards finding an $\epsilon$-approximate Nash equilibrium which is fairly unrelated to the regret objective of PbRL problem we considered in this paper. 

%
Another closely related sub-field of RL, \emph{imitation learning}, addresses the objective of learning optimal behavior from trajectories suggested by an expert. In \cite{ng2000algorithms,boularias2011relative,neu2012apprenticeship,wulfmeier2015maximum}, inverse reinforcement learning problems have been considered, where the objective is to extract (unknown) reward function from the trajectories given by an oracle or expert. Once the reward functions are computed, any RL algorithm could, in principle, be applied to compute the optimal policy. \citet{ho2016generative} propose a generative adversarial network based imitation learning algorithm that computes the optimal policy directly from the trajectories of expert. Our work is fundamentally different in the sense that we do not receive trajectories or optimal actions from an expert. Instead, we get preferences over sample trajectories that are posed as queries to a system expert for preference feedback. 


\section{Problem Setup}
\label{sec:prob}
\vspace{-10pt}

\textbf{Notation. } Let $[n]$ denote the set $\{1,2, \ldots n\}$. Given a set $A$, for any two items $x,y \in A$, we denote that $i$ is preferred over $j$ by $x \succ y$. By $\cB_r(d)$ we denote the $\ell_2$-norm ball of radius $r$ in dimension $d$. 
Lower case bold letters denote vectors, upper case bold letters denote matrices.

\textbf{RL Model.} Consider a $T$-episode, $H$-horizon RL setup $(\mathbb P, \mathcal S, \mathcal A, H, \rho)$, $\cS$ is a finite set of states, $\cA$ is a set of actions, $\mathbb P(\cdot \mid s,a)$ is the MDP transition dynamics given a state and
action pair $(s; a)$, $H \in \N$ is the length of an episode, $\rho$ denotes the initial distribution over states. 

We denote a trajectory by concatenation of all states and actions visited during $H$ steps $\tau:= (s_1, a_1, \cdots , s_{H}, a_H)$. In general let $\tau_{h:H}:= (s_h, a_h, \cdots , s_{H}, a_H)$ denote the states and action from step $h$ until the end of the episode.  We denote by $\tau_h$ to be all the states and actions taken up to step $h$ and define $\tau_0 = \emptyset$.
Let $\Gamma$ be the set of all possible trajectories of length $H$, similarly $\Gamma_h$ denotes the set of all sub-trajectories up to step $h$. We use the  superscript $t$ as in $\tau^t$ to denote a trajectory sampled during the $t-$th episode. At the start of each episode, we assume the initial state $s_1$ is drawn from a fixed distribution $\rho$ known to the learner apriori (for example concentrated on an initial state $s_0$).



\textbf{Trajectory embedding.} For any trajectory $\tau$ we assume the existence of a trajectory embedding function $\phi : \Gamma  \rightarrow \mathbb{R}^d$. We denote by $\phi(\tau)$ to the $d-$dimensional embedding of trajectory $\tau$. The map $\phi$ is known to the learner.
One special case of such a trajectory-dependent feature map is a decomposed embedding, where $\phi(\tau) =  \sum_{h=1}^H  \phi(s_h, a_h)$ and $\phi: \mathcal{S} \times \mathcal{A} \rightarrow \mathbb{R}^d$ is a mapping from state-actions pairs to $\mathbb{R}^d$. Examples of these trajectory embeddings can be borrowed from the Behavior Guided class of algorithms for policy optimization found in~\cite{pacchiano2020learning}. Many practically relevant trajectory or state-action embeddings can be found defined in~\cite{pacchiano2020learning,parker2020effective}. It is conceivable the preference model may be based on one of these embedding maps.

\textbf{Policy embedding.} The above feature embedding also leads to a natural mean embedding of any policy $\pi: \cS \mapsto \cA$ given by $\phi(\pi):=\E_{\tau \sim \pi}[\phi(\tau)]$.
 
\textbf{Preference modeling.} Assuming $\w^* \in \R^d$ to be an unknown vector, we define the pairwise-preference of trajectory $\tau_1$ over $\tau_2$ as: 
\vspace{-5pt}
\begin{align}
\begin{split}
\label{eq:pref_logistic}
 & \mathbb{P}( \tau_1 \succ  \tau_2 )  = \sigma( \langle \phi(\tau_1) - \phi(\tau_2), \w^* \rangle )  \\
 & = \frac{\exp(\phi(\tau_1)^\top \w^*)}{\exp(\phi(\tau_1)^\top \w^*) + \exp(\phi(\tau_2)^\top \w^*)} 
 \end{split}
\end{align}

where $\sigma: \R \mapsto [0,1]$ is the logistic link function, i.e. $\sigma(x) = (1+e^{-x})^{-1}$. We can `lift' the definition of a comparison from trajectories to policies by setting,
\begin{align}
\begin{split}
\label{eq:pref_logistic_policies}
 & \mathbb{P}( \pi_1 \succ  \pi_2 )  = \sigma( \langle \phi(\pi_1) - \phi(\pi_2), \w^* \rangle ) 
 \end{split}
\end{align}
Equation~\ref{eq:pref_logistic} says the probability of any trajectory $\tau_1$ being preferred over $\tau_2$ is essentially proportional to the score difference of the individual trajectories, assuming the score for any trajectory $\tau$ is defined as the function \[s(\tau):= \langle \phi(\tau),\w^*\rangle.\] 
The linear score of any policy $\pi$ (expectation over trajectories) can be similarly defined as $s(\pi):= \E_{\tau \sim \pi}[\langle \phi(\tau),\w^*\rangle]$ and therefore $\mathbb{P}( \pi_1 \succ  \pi_2  ) = \sigma(  s(\pi_1) - s(\pi_2))$.

\textbf{Non markovian policy class. }The performance of all our algorithms will be measured against the policy that maximizes $s(\pi)$. Since $s(\tau) $ may be a non-markovian function of the trajectory, the policy optimizing this objective need not be markovian. We therefore set $\Pi$ as the set of all history dependent policies. In contrast with standard markovian RL works, this is one of the main sources of technical complexity of our setting.


\begin{assumption}
\label{assumption::bounded_parameter}[Bounded parameter]
We assume that $\| \w^*\|\leq W$ for some known $W > 0$.
\end{assumption}

\begin{assumption}
\label{assumption::bounded_features}[Bounded feature maps]
For all trajectories $\tau$ we assume that $\| \phi(\tau)\|\leq B$ for some known $B > 0$.\footnote{Note $B$ could essentially depend on the trajectory-length $H$.}
\end{assumption}

\begin{definition}\label{assump:kappa}
The degree
of non-linearity of the sigmoid $\sigma$ over the parameter space (denoting the first derivative of $\sigma$ by $\sigma'$) is given by 
$$
\kappa:= \sup_{\x \in \cB_B(d), \w \in \cB_S(d)} \frac{1}{{\sigma'}(\w^\top\x)}.
$$
\end{definition}

\textbf{Objective:} Alternative: The objective of the learner is to minimize regret by finding policies to maximize the sum of their expected scores over $T$ rounds. At each round $t$, the learner proposes two policies, $\pi^1_t$ and $\pi^2_t$, which are executed in the MDP generating trajectories $\tau_t^1$ and $\tau_t^2$. The learner then receives feedback in the form of the Bernoulli variable $o_t \in \{0, 1\}$ which specifies whether $\tau^1_t$ is preferred ($o_t = 1$) or $\tau^2_t$ is preferred ($o_t = 0$). The preference feedback $o_t$ is distributed according to $P(\tau_t^1 \succ \tau^2_t)$. We measure the learner's performance via its pseudo-regret w.r.t. policy class $\Pi$, which we define as:
\vspace{-10pt}
\begin{align}
\label{eq:reg1}
\nonumber R_T^{\text{scr}} &:= \max_{\pi \in \Pi}\sum_{t = 1}^T\frac{\big[ (2\phi(\pi) - \phi(\pi_t^1) - \phi(\pi_t^2))^\top \w^* \big] }{2}\\
&= \sum_{t = 1}^T \frac{2s(\pi^*) - \big(s(\pi_t^1)+s(\pi_t^2)\big)}{2},
\end{align}
\vspace{0pt}
where $\pi^*:= \max_{\pi \in \Pi}s(\pi)$. This essentially measures the performance of the learner at round $t$ in terms of average score of the played policies $\pi_t^1, \pi_t^2$ w.r.t. the score maximizing policy $\pi^*$.

\begin{rem}
An important thing to note is that representing any trajectory pair $(\tau_1,\tau_2)$ by the feature $\big(\phi(\tau_1)-\phi(\tau_2)\big) \in \R^d$, our underlying preference model is similar to reward model of \cite{chatterji21} (see Assumption $2.1$). The fundamental difference between our setting and that of~\cite{chatterji21} is the nature of the dueling feedback. In our work, the only way to gather any information when interacting with the world is by comparing the trajectories of two different policies. This adds a layer of complexity not present in the per trajectory feedback model from~\cite{chatterji21}, that makes their algorithms not immediately applicable to our setting. 
\end{rem}

One may think of using our preference model (Equation~\ref{eq:pref_logistic_policies}) to define an alternative notion of regret:
\vspace{-10pt}
\begin{align}
\label{eq:reg2}
R_T^{\text{pref}} := \max_{\pi \in \Pi} \sum_{t = 1}^T \frac{\mathbb{P}(\pi \succ \pi_t^1) + \mathbb{P}(\pi \succ \pi_t^2) - 1}{2}.
\end{align}
\vspace{0pt}
Fortunately, these two notions of regret can be shown to be `equivalent' in the following sense,

 \textbf{Claim 1.} Let $\pi^* \in \argmax_{\pi \in \Pi}s(\pi)$. Then $\pi^*$ also achieves the $\max$ in Eqn. \ref{eq:reg2}.

The logistic link function is increasing w.r.t. its argument, thus for any $\pi, \pi'$ we have $\mathbb{P}(\pi^* \succ \pi' ) = \sigma(s(\pi^*) - s(\pi') )  \geq  \sigma(s(\pi) - s(\pi') ) =   \mathbb{P}(\pi \succ \pi ')$. It follows that for all $t$ and all $\pi$: 
$$\mathbb{P}(\pi^* \succ \pi_t^1 ) + \mathbb{P}(\pi^* \succ \pi_t^2 ) \geq \mathbb{P}(\pi \succ \pi_t^1 ) + \mathbb{P}(\pi \succ \pi_t^2 )  $$
thus establishing the claim.

This argument can also be used to show  $R_T^{\text{scr}}$ and $R_T^{\text{pref}}$ are equivalent up to constant factors when $B,W \leq 1$. The proof is given in Appendix \ref{app:prob}.

\textbf{Claim 2.} $\frac{R_T^{\text{scr}}}{2(e + 1)} \le R_T^{\text{pref}} \le \frac{R_T^{\text{scr}}}{2}$.

We conclude that a strategy that attains sublinear $R_T^{\text{scr}}$ regret also has sublinear $R_T^{\text{perf}}$ regret.

\section{Preference-Based Learning with Known Model}
\label{sec:known-model}
In this section, we introduce and analyze an algorithm for solving the preference-based RL problem when the transition model, $\mathbb P$, that governs the probability of transitioning to a next state is known to the learner. In this case, it becomes possible to directly compute expected features induced by policies; however, the \emph{difficulty of learning based only on preference feedback as opposed to rewards remains}. This is because we have access to feedback only through relative preferences on the trajectories rather than an assumed known reward function. Before stating the algorithm, we first detail a method of estimating the underlying parameter $\w_*$ in the logistic model. This procedure serves as a basis for the algorithm.

\subsection{Maximum Likelihood Estimation}
In the logistic model, a natural way of computing an estimator $\w_t$ of $\w^*$ given trajectory pairs $\{(\tau_\ell^1, \tau_\ell^2 )\}_{\ell=1}^{t-1}$ and preference feedback values $\{ o_\ell \}_{\ell=1}^{t-1}$ is via maximum likelihood estimation. At time $t$ the regularized log-likelihood (or negative cross-entropy loss) of a parameter $\w$ can be written as:
\begin{align*}
    \mathcal{L}_t^\lambda(\w) &= \sum_{\ell=1}^{t=1}\left(o_\ell \log(\sigma( \langle \phi(\tau_\ell^1) - \phi(\tau_\ell^2) \right) , \w \rangle  )) -\frac{\lambda}{2} \| \w \|_2^2\\
     &+(1-o_\ell) \log\left( 1- \mu(  \langle \phi(\tau_\ell^1 ) -\phi(\tau_\ell^2) , \w\rangle )   \right)  ,
\end{align*}
where $\lambda > 0$ is a regularization parameter. 
The function $\mathcal{L}_t^\lambda$ is strictly concave for $\lambda> 0$. The maximum likelihood estimator $\widehat{\w}_t^{\mathrm{MLE}}$ can be written as $\widehat{\w}_t^{\mathrm{MLE}} = \argmax_{\w \in \mathbb{R}^{d}} \mathcal{L}_t^{\lambda}(\w)$. Unfortunately, $\widehat{\w}_t^{\mathrm{MLE}}$ may not satisfy the boundedness Assumption~\ref{assumption::bounded_parameter}, so we instead make use of a projected version of $\widehat{\w}_t^{\mathrm{MLE}}$. Following \cite{faury2020improved}, and recalling Assumption \ref{assump:kappa}, we define a data matrix and a transformation of $\widehat{\w}_t^{\mathrm{MLE}}$ given by
\begin{align*}
    \rmV_t  & = \kappa \lambda \mathbb{I}_d  + \sum_{\ell=1}^{t-1} \left( \phi(\tau_\ell^1) - \phi(\tau_\ell^2)  \right)\left( \phi(\tau_\ell^1) - \phi(\tau_\ell^2) \right)^\top \\
    g_t(\w) & = \sum_{\ell=1}^{t-1} \sigma(\langle \phi(\tau_\ell^1) - \phi(\tau_\ell^2), \w\rangle ) \left( \phi(\tau_\ell^1) - \phi(\tau_\ell^2)\right) + \lambda  \w
\end{align*}
Then, the projected parameter, along with its confidence set, is given by
\begin{align}\label{eq::projection}
    \w_t^L & = \argmin_{\w ~\text{s.t.}~ \|\w \|\leq W} \| g_t(\w) - g_t(\widehat{\w}_t^{\mathrm{MLE}} ) \|_{\rmV_t^{-1}} \\
    \label{eq::c_t}
    \mathcal{C}_t(\delta) & = \{ \w \text{ s.t. } \| \w  - \w^L_t\|_{\rmV_t} \leq 2\kappa \beta_t(\delta) \}
\end{align}


where $\beta_t(\delta) = \sqrt{\lambda} W + \sqrt{\log(1/\delta) + 2d\log\left(1 + \frac{tB}{\kappa \lambda d}\right)}$. We restate a bound by~\cite{faury2020improved} that shows the probability of $\w_\star$ being in $\mathcal{C}_t(\delta)$ for all $t\geq 1$ can be lower bounded.

\begin{lemma}
\label{lemma::confidence_interval_anytime} [Lemma 1 from \cite{faury2020improved}\footnote{A slight modification in the expression of $\beta_t(\delta)$ is needed to incorporate the fact that we assume $\|\phi(\tau)\| \leq B$ for any $\tau$ (Assumption \ref{assumption::bounded_features}) while $B = 1$ in \cite{faury2020improved}. But this can be easily incorporated using Thm. $1$ and Lem. $10$ of \cite{abbasi2011improved} in the final step of the proof of Lem. $12$ of \cite{faury2020improved}}] Let $\delta \in (0,1]$ and define the event that $\w_\star$ is in the confidence interval $\mathcal{C}_t(\delta)$ for all $t \in \mathbb{N}$:
\begin{equation*}
    \mathcal{E}_\delta = \{ \forall t \geq 1,  \w_\star \in \mathcal{C}_t(\delta) \}.
\end{equation*}
Then $\mathbb{P}( \mathcal{E}_\delta ) \geq 1-\delta$. 
\end{lemma}


\subsection{Algorithm and Analysis}

We are now ready to state the Logistic Preference based Reinforcement Learning  (LPbRL) algorithm with known model, shown in Algorithm~\ref{alg:rrl}. Before any interaction or feedback, we initialize identical data matrices $\V_1 = \overline \V_1 = \kappa\lambda \I_d$, $\lambda > 0$ being a regularization parameter. $\V_t$, as defined before, is designed to track the exact covariates used in the maximum likelihood estimation. $\overline \V_t$ (Line~\ref{line::vtbar}) on the other hand tracks a similar quantity, but instead uses the expected features under a given policy. 

At each round $t$, we then compute an estimate $\w_t^L$ and determine a set of candidate policies $\SS_t$ for which no other policy $\pi$ significantly outperforms a member of $\SS_t$. The threshold for what constitutes ``significant'' is determined by the uncertainty in the estimate of $\w_t^L$. We then search over this set to identify two policies, $\pi_t^1$ and $\pi_t^2$, with expected features that maximize the uncertainty determined by $\overline \V_t$, precisely by choosing $(\pi_t^1, \pi_t^2) = \arg\max_{ \pi^1, \pi^2 \in \Pi_t} \| \phi(\pi^1) - \phi(\pi^2) \|_{\overline \V_t^{-1}}$. Both policies are deployed, inducing trajectories $\tau_t^1$ and $\tau_t^2$ and feedback $o_t$ is received. We then update the data matrices $\V_t$ and $\overline \V_t$ with the trajectory features $\phi(\tau_t^1 ) - \phi(\tau_t^2)$ and expected features $\phi(\pi_t^1) - \phi(\pi_t^2)$, respectively. The procedure is repeated for each round $t \in [T]$.

\begin{restatable}[]{theorem}{regk}
\label{thm::regret-known}
Let $\delta \leq 1/e$ and $\lambda \geq \frac{B}{\kappa}$. Then, with probability at least $1 - \delta$, the expected regret of Algorithm~\ref{alg:rrl} can be bounded by
\begin{equation*}
   R_t \leq \left(4\kappa \beta_t(\delta) + 2\alpha_{d,T}(\delta)  \right) \sqrt{2Td \log\left( 1+ \frac{TB}{ \kappa d}\right) } .
\end{equation*}
\end{restatable}

Note there is no dependence on the size of the state or action spaces on account of the model being known in this setting. Furthermore,  we note that any dependence on the horizon $H$ is effectively accounted for in the size of the constant $B$ that bounds the norm of the trajectory features $\phi(\tau)$. For example, if $\phi(\tau)$ decomposes in a per-timestep fashion as $\phi(\tau) = \sum_{h \in [H]} \phi(s_h, a_h)$ where each $h$ satisfies $\| \phi(s_h, a_h)\| \leq B'$, then a trivial bound would give $B \leq B' H$. However, Assumption~\ref{assumption::bounded_features} allows for greater generality.
\begin{algorithm}[h]
   \caption{\textbf{\algrrl: Regret minimization (Known Model)}}
   \label{alg:rrl}
\begin{algorithmic}[1]
  \STATE {\bfseries input:} Regularization parameter $\lambda$, Learning rate $\eta_t >0$, exploration length $t_0>0$
  \STATE Define $\alpha_{d,T}(\delta) = 20 B W\sqrt{d \log(T(1 + 2 T)/\delta)} $ and $\gamma_t(\delta) =2\kappa\beta_t(\delta) + \alpha_{d, T}(\delta) $.
  \STATE Initialize $\overline \V_t = \kappa \lambda \mathbb{I}_d$ 
  \FOR{$t = 1,2,\ldots T$}
    {\small
    \STATE Compute $\w_t^L$ (see Eqn. \eqref{eq::projection})
    }
    \STATE Set $\SS_t = \{ \pi^1 | (\phi(\pi^1) - \phi(\pi))^\top \w_t^L + $
$$\gamma_t(\delta)  \| \phi(\pi^1) - \phi(\pi) \|_{\overline{\mathbf{V}}_t^{-1}} \geq 0 ~\forall \pi \}$$
\vspace{-.5cm}
\STATE Compute $$(\pi_t^1, \pi_t^2) = \arg\max_{ \pi^1, \pi^2 \in \Pi_t} \| \phi(\pi^1) - \phi(\pi^2) \|_{\overline \V_t^{-1}}.$$
\vspace{-.3cm}
    \STATE Sample $\tau_t^1 \sim \pi_t^1$ and $\tau_t^2 \sim \pi_t^2$.
    \STATE Play the duel $(\tau_t^1,\tau_t^2)$ and receive $o_t = \1(\tau_t^1 \text{ beats } \tau_t^2)$
    \STATE Update $$\overline \V_{t+ 1} = \overline \V_t + (\phi(\pi_t^1)-\phi(\pi_t^2 ))(\phi(\pi_t^1)-\phi(\pi_t^2))^\top$$ \label{line::vtbar}
    \vspace{-.4cm}
  \ENDFOR
\end{algorithmic}
\end{algorithm}

\begin{rem}
Theorem~\ref{thm::regret-known} shows that for a sufficiently large choice of the regularization parameter $\lambda$, the pseudo-regret of Algorithm~\ref{alg:rrl} is at most $R_t = {\mathcal{O}} \left( \left( W\sqrt{ \kappa B}   + W B\right)  d  \log (TB /\kappa \delta) \sqrt{T}  \right)$. Importantly, the regret  scales nearly optimally with $ d \sqrt T$ dependency given existing $\Omega(d\sqrt{T})$ lower bounds for linear bandits \citep{lattimore2020bandit} and known reductions between the standard and preference regret \cite{saha2021optimal}.  Assuming $\kappa$ to be constant, we pay the additional factors in $B$ and $W$ due to non-Markovian rewards which are only indirectly revealed to the learner in terms of preferences. 
\end{rem}

\subsection{Regret Analysis: Proof Sketch of Thm. \ref{thm::regret-known}}
We now sketch the proof of Theorem~\ref{thm::regret-known}. Details and proofs of supporting results can be found in Appendix~\ref{sec::known-cor-proof}. The \emph{main idea} of the proof is to ensure that $\Pi_t$ contains only candidate policies that are predicted to be ``sufficiently good'' under the learned model $\w_t^L$ using the size of the confidence set $\mathcal C_t(\delta)$. We must also verify that $\SS_t$ always contains the optimal policy $\pi^*$. Thus, as long as the set $\mathcal C_t(\delta)$ shrinks at a sufficiently fast rate, our algorithm will have sublinear regret.

However, in order to judge the uncertainty in predictions of the expected value $\phi(\pi)^\top \w_t^L$ of a policy $\pi$, we must relate the data matrix $\V_t$ that controls the accuracy of the learned parameter $\w_t^L$ (see Lemma~\ref{lemma::confidence_interval_anytime}), and its expected counterpart $\overline \V_t$ (used to define $\Pi_t$). The set $\Pi_t$ is characterized via $\overline \V_t$ because this way it allows us to relate it to the algorithm's regret, a quantity that depends on the expected features of the played policies.  Corollary~\ref{corollary::V_bar_bound} establishes that distances $\|\w_t^L - \w^* \|_{\overline \V_t}$ weighted by $\overline \V_t$ are not too far from the same distances $\|  \w_t^L - \w^* \|_{ \V_t}$ weighted by $\V_t$. Let
\begin{align*}
    \mathcal E_{prec} = \left\{ \overline \V_T  \preceq  2\V_T + 84 B^2 d \log((1 + 2 T)/\delta) \mathbb{I}_d \right\}.
\end{align*}



\begin{restatable}{corollary}{corollaryconfidenceinterval}
\label{corollary::V_bar_bound}
Under Assumption~\ref{assumption::bounded_parameter}, conditioned on event $\mathcal E_\delta \cap \mathcal E_{prec}$, for any $t \in [T]$ 
\begin{align*}
    \| \w^* - \w^L_t \|_{  \overline \V_t } & \leq 4 \kappa \beta_t(\delta) + \alpha_{d, T}(\delta),
\end{align*}
where  $\alpha_{d,T}(\delta) = 20 B W\sqrt{d \log(T(1 + 2 T)/\delta)}$.  
Furthermore, if $\delta \leq 1/e$, then $\mathbb P (\mathcal E_\delta \cap \mathcal E_{prec}) \geq 1 - \delta - \delta\log_2T$.
\end{restatable}


The proof of above is given in Appendix \ref{sec::known-theorem-proof}. Leveraging this relationship, we can establish that the confidence set of policies $\Pi_t$ defined in line 5 of Algorithm~\ref{alg:rrl} will contain the optimal policy.

\begin{restatable}{lemma}{lemcontainsoptimal}\label{lem::contains-optimal}
Conditioned on event $\mathcal E_\delta \cap \mathcal E_{prec}$, $\pi^* \in \SS_t$,
\end{restatable}



The remainder of the proof now consists of showing the instantaneous regret can be bounded in terms of the size of the confidence sets and the uncertainty values $\| \phi(\pi^1_t) - \phi(\pi^2_t) \|_{\overline \V_t}$. We defer the final details to Appendix~\ref{sec::proof-known-theorem}.

\section{Unknown model: Algorithm and Analysis}
\label{sec:ub}
\label{sec:unknown-model}

\textbf{Algorithm description. }
%
The LPbRL algorithm for unknown dynamics models works in a similar way to Algorithm~\ref{alg:rrl}. The main differences lay in the definition of the set $\Pi_t$. Whereas in Algorithm~\ref{alg:rrl} this set of policies can be defined without taking into account the model uncertainty, in this case the set of policies to optimize over needs to be carefully constructed in such a way that it can be shown to contain $\pi_\star$ (see Lemma~\ref{lem::contains-optimal_unknown}). With this in mind we start by introducing the necessary technical tools that will be used throughout this section to deal with model uncertainty.

\subsection{Analysis of instantaneous regret:}

For any policy $\pi$ and any MDP model $\mathbb{P}$, we denote by $\phi^{\mathbb{P}}(\pi)$ to the mean feature of policy $\pi$ in model $\mathbb{P}$. We use the notation $N_{t}(s,a)$ to denote the number of samples of action $a$ at state $s$ the learner has collected up to time $t$. We use the notation $\widehat{\mathbb{P}}_t$ to denote the empricial model at time $t$. We use an 'empirical' version of $\overline{\mathbf{V}}_t$ defined using the average features computed using the model available at time $t$:
\begin{small}
\begin{equation}
    \widetilde{\rmV}_t    = \kappa \lambda \mathbb{I}_d  + \sum_{\ell=1}^{t-1} \left( \phi^{\widehat{\mathbb{P}}_\ell}(\pi_\ell^1) - \phi^{\widehat{\mathbb{P}}_\ell}(\pi_\ell^2)  \right)\left( \phi^{\widehat{\mathbb{P}}_\ell}(\pi_\ell^1) - \phi^{\widehat{\mathbb{P}}_\ell}(\pi_\ell^2) \right)^\top  
\end{equation}
\end{small}
Our confidence intervals will use a Mahalanobis norm defined by this covariance matrix. Throughout this section we will make heavy use of some of the results from~\cite{chatterji21}. With that in mind we will define a variety of bonus terms. Given any $\eta > 0$ define,
\begin{align*}
\pushleft{\xi^{(t)}_{s,a}(\eta, \delta) = \min\left(2\eta, 4\eta \sqrt{\frac{U}{N_t(s,a)} }\right)  }  \\
\text{s.t. } U = H\log(|\mathcal{S}||\mathcal{A}|) + \log\left( \frac{6\log(N_t(s,a))}{\delta}\right).
\end{align*}
 We define the following `bonus' function corresponding to the expectation of these bonus terms summed over a trajectory sampled from a policy $\pi$ in the model $\hat{\mathbb{P}}_t$,
\begin{equation*}
    \widehat{B}_t(\pi, \eta, \delta) =  \mathbb{E}_{s_1 \sim \rho, \tau \sim \hat{\mathbb{P}}_t^{\pi}(\cdot | s_1)} \left[ \sum_{h=1}^{H-1} \xi_{s_h, a_h}^{(t)}(\eta, \delta)  \right].
\end{equation*}



Similar to the previous theorem, we must relate $\| \w^* - \w_t^L\|_{\widetilde{\mathbf{V}}_t} $ and $\| \w^* - \w_t^L\|_{\overline{\mathbf{V}}_t} $. We do this via a series of Lemmas.

\begin{restatable}{lemma}{lemmaupperboundupperboundingwlnormunknownone}\label{lemma::upper_bound_upper_bounding_w_L_norm_unknown1}
Let $\bar{\mathcal{E}}_0$ be the event that for all $t \in \mathbb{N}$,
\begin{align*}
       \| \w_t^L - \w_* \|_{\widetilde{\rmV}_t}  \leq \sqrt{2}\| \w_t^L - \w_* \|_{\overline{\rmV}_t} +      \qquad  \qquad
       &
       \\ \sqrt{\sum_{\ell=1}^{t-1} 4\left(\widehat{B}_t\left(\pi, 2WB, \frac{\delta'}{8\ell^3|\mathcal{A}|^{\mathcal{S}} }\right)\right)^2 } + \frac{1}{t}.
\end{align*}
where $\delta'= \frac{\delta}{\left( \frac{1+4W}{\epsilon}\right)^d}$ and $\epsilon = \frac{1}{t^2\kappa \lambda + 4B^2t^3 }$. Then $\mathbb{P}\left( \bar{\mathcal{E}}_0 \right) \geq 1-\delta$.
\end{restatable}


The proof of Lemma~\ref{lemma::upper_bound_upper_bounding_w_L_norm_unknown1} is in Appendix~\ref{section::proof_upper_bound_upper_bounding_w_L_norm_unknown1}. We now proceed to define the set $\Pi_t$. To do so, it will be useful to introduce the following confidence radius multiplier
\begin{small}
\begin{align*}
  \pushleft{  \gamma_t = \sqrt{2} \left( 4\kappa \beta_t(\delta) + \alpha_{d, T}(\delta)  \right)   + \frac{1}{t} + }\\
    2\sqrt{\sum_{\ell=1}^{t-1} \widehat{B}^2_t\left(\pi_\ell^1, 2WB, \frac{\delta'}{8\ell^3|\mathcal{A}|^{\mathcal{S}} }\right) + \widehat{B}^2_t\left(\pi_\ell^2, 2WB, \frac{\delta'}{8\ell^3|\mathcal{A}|^{\mathcal{S}} }\right)  }.
\end{align*}
 \end{small}
Finally,
\begin{align*}
\Pi_t &= \Bigg\{ \pi^1 \Big | (\phi^{\widehat{\mathbb{P}}_t} (\pi^1) - \phi^{\widehat{\mathbb{P}}_t}(\pi) )^\top \w_t^L +  \\
&\qquad \gamma_t  \| \phi^{\widehat{\mathbb{P}}_t} (\pi^1) - \phi^{\widehat{\mathbb{P}}_t}(\pi) \|_{\widetilde{\mathbf{V}}^{-1}_t} +\widehat{B}_t\left(\pi^1, 2SB, \frac{\delta}{2|\mathcal{A}|^{\mathcal{S}}}\right)     \\
&\qquad  + \widehat{B}_t\left(\pi, 2SB, \frac{\delta}{2|\mathcal{A}|^{\mathcal{S}}}\right)\geq 0, \forall \pi \Bigg \}.
\end{align*}
\begin{algorithm}[H]
   \caption{\textbf{\algrrl: Regret minimization (Unknown Model)}}
   \label{alg:urrl}
\begin{algorithmic}[1]
  \STATE {\bfseries input:} Learning rate $\eta_t >0$, exploration length $t_0>0$
  \STATE Initialize empirical model $\widehat{\mathbb{P}}_1$.
  \FOR{$t = 1,2,\ldots T$}
    {\small
    \STATE Compute $\w_t^L$ and $\Pi_t$.\\
    }
    \STATE Compute 
    \vspace{-.3cm}
    \begin{align*}
        (\pi_t^1, \pi_t^2) &= \argmax_{  \pi^1, \pi^2 \in \Pi_t} \gamma_t \| \phi^{\widehat{\mathbb{P}}_t}(\pi^1) - \phi^{\widehat{\mathbb{P}}_t}(\pi^2) \|_{\widetilde{\mathbf{V}}_t^{-1}} + \\
        &\quad 2\widehat{B}_t(\pi^1, 2WB, \delta)  + 2 \widehat{B}_t(\pi^2, 2WB, \delta)
    \end{align*}
    \vspace{-.5cm}
    \STATE Sample $\tau_t^1 \sim \pi_t^1$ and $\tau_t^2 \sim \pi_t^2$.
    \STATE Play the duel $(\tau_t^1,\tau_t^2)$ and receive $o_t = \1(\tau_t^1 \text{ beats } \tau_t^2)$
    \STATE Update 
    \vspace{-.5cm}
    \begin{small}
    $$\widetilde{\mathbf{V}}_{t+1} = \widetilde{\mathbf{V}}_{t} + \left( \phi^{\widehat{\mathbb{P}}_\ell}(\pi_\ell^1) - \phi^{\widehat{\mathbb{P}}_\ell}(\pi_\ell^2)  \right)\left( \phi^{\widehat{\mathbb{P}}_\ell}(\pi_\ell^1) - \phi^{\widehat{\mathbb{P}}_\ell}(\pi_\ell^2) \right)^\top  $$
    \end{small}
  \STATE Update empirical model and build $\widehat{\mathbb{P}}_{t+1}$.
  \ENDFOR
  
\end{algorithmic}
\end{algorithm}
\vspace{-10pt}

Algorithm~\ref{alg:urrl} shares the structure of Algorithm~\ref{alg:rrl}. The main difference lies in the definition of $\Pi_t$ and in the optimization problem to find $(\pi_t^1, \pi_t^2)$. We can prove a result similar to Lemma~\ref{lem::contains-optimal} and show that $\pi^* \in \Pi_t$. 
\begin{restatable}{lemma}{lemcontainsoptimalunknown}\label{lem::contains-optimal_unknown}
Let $\bar{\mathcal{E}}_{-1}$ be the event that $\pi^* \in \Pi_t$ for all $t \in \mathbb{N}$. Then $\mathbb{P}\left( \bar{\mathcal{E}}_{-1} \right) \geq 1-5\delta$.
\end{restatable}
The proof of Lemma~\ref{lem::contains-optimal_unknown} can be found in Appendix~\ref{section:proof_lemma_pi_star_containment}. The next step in the proof is to exhibit a bound on the instantaneous regret,
\begin{lemma}\label{lemma::first_unknown_model_lemma}
Let $\bar{\mathcal{E}}_2$ be the event that for all $t \in \mathbb{N}$,
\begin{small}
\begin{align*}
 2r_t  &\leq  (\phi^{\widehat{\mathbb{P}}_t}(\pi^*) - \phi^{\widehat{\mathbb{P}}_t}(\pi_t^1))^\top \w^* + (\phi^{\widehat{\mathbb{P}}_t}(\pi^*) - \phi^{\widehat{\mathbb{P}}_t}(\pi_t^2))^\top \w^* + \notag \\
 &\quad \widehat{B}_t(\pi^*, 4WB, \delta) +  \widehat{B}_t(\pi_t^1, 2WB, \delta)  +  \widehat{B}_t(\pi_t^2, 2WB, \delta). 
\end{align*}
\end{small}
 Then $\mathbb{P}\left( \bar{\mathcal{E}}_2 \right) \geq 1-2\delta$.
\end{lemma}

\begin{proof} Note that we can write:
\begin{small}
\begin{align*}
 2r_t & = (\phi(\pi^*) - \phi(\pi_t^1))^\top \w^* + (\phi(\pi^*) - \phi(\pi_t^2))^\top \w^*\\
   &=  (\phi^{\widehat{\mathbb{P}}_t}(\pi^*) - \phi^{\widehat{\mathbb{P}}_t}(\pi_t^1))^\top \w^* + (\phi^{\widehat{\mathbb{P}}_t}(\pi^*) - \phi^{\widehat{\mathbb{P}}_t}(\pi_t^2))^\top \w^* +\\
   &\quad  2( \phi(\pi^*) - \phi^{\mathbb{P}_t}(\pi^*))^\top \w^* +\\
   &~\quad (\phi^{\mathbb{P}_t}(\pi_t^1) -  \phi(\pi_t^1))^\top \w^* + (\phi^{\mathbb{P}_t}(\pi_t^2) -  \phi(\pi_t^2))^\top \w^*
\end{align*}
\end{small}
By Lemma~\ref{lemma::trajectory_score_concentration1} (Lemma B.1 in~\cite{chatterji21}), we conclude that with probability at least $1-\delta$, for all $t \in \mathbb{N}$, setting $\eta = 4 W B$, 
\begin{align*}
    2( \phi(\pi^*) - \phi^{\mathbb{P}_t}(\pi^*))^\top \w^* \leq  \widehat{B}_t(\pi^*, 4WB, \delta)  
\end{align*}
Similarly, as a consequence of Lemma~\ref{lemma::trajectory_score_concentration1} and a union bound, setting $\eta = 2WB$, with probability at least $1-2\delta$ 
\begin{align*}
    (\phi^{\mathbb{P}_t}(\pi_t^1) -  \phi(\pi_t^1))^\top \w^*  \leq \widehat{B}_t(\pi_t^1, 2WB, \delta)  \\
    (\phi^{\mathbb{P}_t}(\pi_t^2) -  \phi(\pi_t^2))^\top \w^*  \leq \widehat{B}_t(\pi_t^2, 2WB, \delta)
\end{align*}
The result follows. 
\end{proof}

Armed with the results of Lemma~\ref{lem::contains-optimal_unknown} we can show the following bound for the regret. 

\begin{restatable}{lemma}{lemmaregretboundunknownsupport}\label{lemma::regret_bound_unknown_support}
With probability at least $1-15\delta$ the regret is bounded by,
\begin{align*}
    R_T &\leq 2\gamma_T \sqrt{2Td\log\left( 1 + \frac{TB}{d}\right) } + \\
    &\quad \sum_{t \in [T]} 4\widehat{B}_t(\pi_t^1, 4WB, \delta) + 4\widehat{B}_t(\pi_t^2, 4WB, \delta)
\end{align*}
\end{restatable}
The proof of Lemma~\ref{lemma::regret_bound_unknown_support} can be found in Appendix~\ref{section::proof_regret_bound_unknown_support}. The derivation follows from a repeated use of the instantaneous regret upper bound derived from Lemma~\ref{lemma::first_unknown_model_lemma}.

The rest of the proof is dedicated to bound the $\widehat{B}_t( \cdot )$ terms. The general idea is to relate these bonus expectations under the empirical model with an expected sum of bonus terms under the true model and sampled according to policies $\pi_t^1$ and $\pi_t^2$. Once this is achieved we have reduced the problem to bound a sum of vanishing markovian errors under the sampling distribution defined by the policies that were selected during optimization. This can be done via a similar argument as many existing RL works. Finally, we also show the $\gamma_T$ term can be bounded by a term of the form $\widetilde{\mathcal{O}}(\kappa \beta_t(\delta) + \alpha_{d,T}(\delta) + \mathrm{poly}(H, |\mathcal{S}|, |\mathcal{A}|) ) $, hides logarithmic factors in $\delta, |\mathcal{S}|$ and $|\mathcal{A}|$.  A detailed discussion of these arguments can be found in Appendix~\ref{section::proofs_unknown_model_all}. Our final main result (simplified) is thus,

\begin{restatable}[]{theorem}{thmunknownsimple}
\label{theorem:main_unknown_model_simple}
The regret of $\mathbf{LPbRL}$ satisfies, 
\begin{small}
\begin{align*}
    R_T &\leq \widetilde{\mathcal{O}}(\kappa d\sqrt{T} + H^{3/2} \sqrt{ |\mathcal{S}| |\mathcal{A}| dTH } + H|\mathcal{S}| \sqrt{ |\mathcal{A}|dTH} ).
\end{align*}\end{small}
 For all $T \in \mathbb{N}$ simultaneously with probability at least $1-15\delta$. 
Where $\widetilde{\mathcal{O}}$ hides logarithmic factors in $\delta, |\mathcal{S}|$ and $|\mathcal{A}|$.
\end{restatable}
The complete version of Theorem~\ref{theorem:main_unknown_model_simple} can be found in Appendix~\ref{section::proofs_unknown_model_all}. Similar to Theorem~\ref{thm::regret-known}, the leading term in the regret scales as $\widetilde {\mathcal O}(d\sqrt{T})$ due to estimation based on the preferences. In addition to this, we now have dependence on $|\mathcal S|$ and $|\mathcal A|$ unlike before.  These arise due to the tabular nature of the problem since the transition dynamics are unknown in this case.


\section{Discussions and Future Scopes}
\label{sec:concl}

In this work we addressed the problem of reinforcement learning from relative preference feedback where the agent does not get to see the absolute reward of actions taken at each state but instead observes the relative preferences between trajectories.
We modeled the preference feedback in terms of the underlying non-Markovian linear reward model and proposed  algorithms for both known as well as unknown MDP transition models. Precisely the regret guarantees of our proposed algorithms are analyzed to be respectively 
 $\widetilde{\mathcal O}\left(  d  \log (T / \delta) \sqrt{T}  \right)$ and
$\widetilde{\mathcal O}((\sqrt{d} + H^2 + |\mathcal{S}|)\sqrt{dT} +\sqrt{|\mathcal{S}||\mathcal{A}|TH} )$
for the case of known and unknown transition models.

%
As discussed in the introduction, preference-based reinforcement learning has applications in several fields including training robots, stock market, recommender systems, two player games, chatbot interactions, etc. Thus there are plenty of scopes to extend the above setup to incorporate the corresponding system requirements, e.g. generalizing dueling trajectory preferences to subsets, considering alternative preference feedback without assuming an underlying reward model, extending to infinite horizon settings with more complex state-actions spaces, etc. Analyzing the fundamental performance limits of the PbRL regret minimization problem and designing algorithms with tighter performance guarantees would also be another interesting direction to investigate.

\section*{Acknowledgment}
\vspace{-10pt}
AS gratefully thanks Aditya Gopalan and Raghuram Bharadwaj Diddigi (IISc Bangalore) for the initial discussions on preference based reinforcement learning literature.  

\newpage

\bibliographystyle{plainnat}
\bibliography{rrl}

\newpage
\appendix
\onecolumn
{
\allowdisplaybreaks

\tableofcontents

\section*{\centering \Large{Supplementary for \papertitle}}

\section{Appendix for Section \ref{sec:prob}}
\label{app:prob}

\textbf{Claim 2.} $\frac{R_T^{\text{scr}}}{2(e + 1)} \le R_T^{\text{pref}} \le \frac{R_T^{\text{scr}}}{2}$.

\subsection{Proof of Claim $2$}

\begin{proof}
Recall by \eqref{eq:reg1}, \eqref{eq:reg2} and Claim $1$.

\begin{align*}
\nonumber R_T^{\text{scr}} = \sum_{t = 1}^T \frac{2s(\pi^*) - \big(s(\pi_t^1)+s(\pi_t^2)\big)}{2},
\end{align*}

\begin{align*}
R_T^{\text{pref}} := \sum_{t = 1}^T \frac{P(\pi^* \succ \pi_t^1) + P(\pi^* \succ \pi_t^2) - 1}{2}.
\end{align*}

Now assume $S, B < 1$. Then for any two policies $\pi_1$ and $\pi_2 \in \Pi$, such that $s(\pi_1) \ge s(\pi_2)$, we have:

\begin{align*}
P & (\pi_1,\pi_2) - 1/2 = \frac{e^{s(\pi_1)} - e^{s(\pi_2)}}{2(e^{s(\pi_1)} + e^{s(\pi_2)})}\\
& = \frac{e^{s(\pi_1)-s(\pi_2)} - 1}{2(e^{s(\pi_1)-s(\pi_2)} + 1)}\\ 
& > \frac{s(\pi_1)-s(\pi_2)}{2(e+1)}
\end{align*}

On the other hand denoting $x = {s(\pi_1)-s(\pi_2)} \in (0,1)$  we get: 
\begin{align*}
& P(\pi_1,\pi_2) - 1/2  = \frac{e^{x} - 1}{2(e^{x} + 1)} < \frac{(x + x^2/2! + x^3/3! + \ldots) }{4} \\
& < \frac{x\big(1 + x/2 + x^2/2^2 + \ldots \big) }{4} < x/2
\end{align*}

The claim now follows combining the above two inequalities and noting that by definition $\pi^*:= \arg\max_{\pi \in \Pi}s(\pi)$. 
\end{proof}
}

\section{Appendix for  Section~\ref{sec:known-model}}
\label{sec::known-theorem-proof}

By first-order optimality conditions,  $\widehat{\w}_t^{\mathrm{MLE}}$ is the point in $\mathbb{R}^d$ satisfying: 
\begin{align*}
    \nabla_{\w} \mathcal{L}_t^\lambda(\widehat{\w}_t^{\mathrm{MLE}}) = \sum_{\ell=1}^{t-1} o_\ell \left( \phi(\tau_\ell^1) - \phi(\tau_\ell^2) \right)\qquad \qquad \qquad \\
    - \left(  \sum_{\ell=1}^{t-1} \sigma(\langle \phi(\tau_\ell^1) - \phi(\tau_\ell^2)  ,    \w\rangle \left( \phi(\tau_\ell^1) - \phi(\tau_\ell^2)\right) + \lambda \w    \right).
\end{align*}

\subsection{Proof of Corollary~\ref{corollary::V_bar_bound}}
\label{sec::known-cor-proof}

The primary mechanism behind Corollary~\ref{corollary::V_bar_bound} is the following lemma for matrix concentration.

\begin{lemma}\label{lem::matrix-concentration}
Let $\delta \leq e^{-1}$.  Then, with probability $1 - \delta \log_2 T$,  for all $t \in [T]$, it holds that
\eq{
\| \w^* - \w^L_t \|^2_{  \overline \V_t } \leq 2 \| \w^* - \w^L_t \|^2_{ \V_t } + 84 B^2 d \log(T(1 + 2 T)/\delta) \| \w^* - \w^L_t \|^2_2
}
\end{lemma}
\begin{proof}
Fix $\v \in \R^d$ such that $\| \v\|_2 = 1$. For $\ell \in [T]$, let $X_\ell = \v^\top \left(\phi(\tau^1_\ell) - \phi(\tau^2_\ell)\right)\left(\phi(\tau^1_\ell) - \phi(\tau^2_\ell)\right)^\top \v$. Furthermore define $X_0 = \kappa \lambda$. Observe that $X_\ell - \E_{\ell - 1} X_\ell$ for $\ell = 0, \ldots, T$ is an $\{\mathcal F_\ell\}$-adapted martingale difference sequence where $\E_{\ell} [ \cdot ] $ denotes the conditional expectation $\E [ \cdot \ | \ \mathcal F_{\ell}]$.

Note that the conditional variance of the individual terms may be bounded above by
\eq{
\text{var}\left( X_\ell \right) &  = \E_{\ell - 1} \left[ X_\ell^2 -\E_{\ell - 1} \left[ X_\ell \right]^2 \right] \\
& \leq \E_{\ell - 1} \left[ X_\ell^2 \right] \\
& \leq 4B^2 \E_{\ell - 1}\left[ X_\ell \right]
}
where we have used the fact that $X_\ell$ is non-negative and $\| \phi(\tau) \|_2 \leq B$.  Let $\widehat \V_t = \kappa \lambda \mathbb I +  \sum_{s \in [t]}  \E_{\ell - 1}\left( \phi(\tau^1_\ell) - \phi(\tau^2_\ell) \right) \left( \phi(\tau^1_\ell) - \phi(\tau^2_\ell) \right)^\top$.

By \cite[Lemma 2]{bartlett2008high}, we have, with probability at least $1 - \delta \log_2 T$, 
\eq{
 \v^\top \widehat \V_T \v =  \sum_{\ell = 0}^T  \E_{\ell - 1} X_\ell & \leq \sum_{\ell = 0}^T  X_{\ell} +  \sqrt{ 16\log(1/\delta) \sum_{\ell = 0}^T \text{var}_{\ell - 1}(X_{\ell}) } + 2B^2 \log(1/\delta)   \\
& \leq \sum_{\ell = 0}^T  X_{\ell} +  \sqrt{ 64 B^2 \log(1/\delta) \sum_{\ell = 0}^T \E_{\ell -1 }  X_{\ell} } + 2B^2 \log(1/\delta)  \\
& \leq  \sum_{\ell = 0}^T  X_{\ell}  + \frac{1}{2}\sum_{\ell = 0}^T \E_{\ell - 1}  X_{\ell} +  34 B^2 \log(1/\delta) \\
& =  \v^\top  \V_T \v + \frac{1}{2} \v^\top \widehat \V_T \v  + 34 B^2 \log(1/\delta)
}
where the third line applied the AM-GM inequality. Rearranging shows that
\eq{
\frac{1}{2} \v^\top \widehat \V_T \v & \leq \v^\top  \V_T \v  + 34 B^2 \log(1/\delta)
}
This holds for a fixed $\v$. We now show that it approximately holds for all $\v$ such that $\| \v\|_2 = 1$ via a covering argument. 

Let $\CC$ be a minimal $\epsilon$-cover of $\SS^{d - 1} = \{ \v \in \R^d \ : \ \| \v\|_2 = 1\}$. A standard result states that $|\CC_\epsilon | \leq (1 + 2/\epsilon)^d$. Then, by the union bound, with probability $1 - \delta \log_2 T$, for all $ \v\in \CC_\epsilon$,
\eq{
\frac{1}{2} \v^\top \widehat \V_T \v & \leq \v^\top  \V_T \v  + 34 B^2 d \log((1 + 2/\epsilon)/\delta)
}
Let $A = \frac{1}{2}\V_T - \widehat \V_T$. Note that $\| A\| \leq 4B^2 T$ by definition. Let $\v \in \SS^{d - 1}$ be arbitrary and let $u_v \in \CC_\epsilon$ be the closest vector in the cover so that $\| \v - \u\|_2 \leq \epsilon$. Then, 
\eq{
\v^\top A \v & = \v^\top A \v + \u^\top A \u - \u^\top A \u  \\
& \leq \u^\top A \u + 8\epsilon B^2 T \\
& \leq 34 B^2 d \log((1 + 2 T)/\delta) + 8 B^2 \\
& \leq 42 B^2 d \log((1 + 2 T)/\delta)
}
under the good event and choosing $\epsilon = 1/T$. Since this holds for all $\v \in \SS^{d - 1}$, we conclude that 
\eq{
\widehat \V_T  \preceq  2\V_T + 84 B^2 d \log((1 + 2 T)/\delta) \mathbb{I}_d
}
with probability at least $1 -\delta \log_2T$. Finally, by Jensen's inequality we have $\overline \V_T  \preceq \widehat \V_T$. Then, we apply the union bound over $t \in [T]$, which gives the result.
\end{proof}

The proof of the corollary now follows immediately as a consequence.

\corollaryconfidenceinterval*

\begin{proof}[Proof of Corollary~\ref{corollary::V_bar_bound}]
Assuming that $\mathcal E_\delta$ holds, we have that $\| \w^* - \w_t^L \|_{\V_t} \leq 2\kappa \beta(\delta)$. Furthermore, Lemma~\ref{lem::matrix-concentration} gives
\begin{align*}
    \| \w^* - \w^L_t \|_{  \overline \V_t } & \leq  \sqrt 2 \| \w^* - \w^L_t \|_{ \V_t } + 10 B \sqrt {d \log(T(1 + 2 T)/\delta) }  \| \w^* - \w^L_t \|_2 \\
    & \leq 4 \kappa \beta_t(\delta) + 20 B S \sqrt {d \log(T(1 + 2 T)/\delta) } 
\end{align*}
Also, $\mathbb P (\mathcal E_\delta \cap \mathcal E_2) \geq 1 - \delta - \delta\log_2T$ follows from above combined with the claim of Lemma \ref{lemma::confidence_interval_anytime}.
\end{proof}

\subsection{Proof of Lemma~\ref{lem::contains-optimal}}
\lemcontainsoptimal*
\begin{proof}
Condition on $\mathcal{E}_\delta \cap \mathcal E_{prec}$. By definition of $\pi^*$, we have $\left(\phi(\pi^*) - \phi(\pi)\right)^\top \w^* \geq 0$ for any arbitrary $\pi$. This implies
\begin{align*}
    0 & \leq \left(\phi(\pi^*) - \phi(\pi) \right)^\top \w_t^L + \| \phi(\pi^*) - \phi(\pi) \|_{\overline \V_t^{-1}} \cdot \| \w^* - \w_t^L \|_{\overline \V_t} \\
    & \leq\left(\phi(\pi^*) - \phi(\pi) \right)^\top \w_t^L + \left( 4\kappa\beta_t(\delta) + \alpha_{d, T}(\delta)\right) \cdot \| \phi(\pi^*) - \phi(\pi) \|_{\overline \V_t^{-1}}
\end{align*}
where the second line follows from Corollary~\ref{corollary::V_bar_bound}.
\end{proof}

\subsection{Proof of Theorem~\ref{thm::regret-known}}\label{sec::proof-known-theorem}

We require a standard determinant bound to complete the proof.

\begin{lemma}\label{lem::determinant}
Let $\lambda \geq B$. Consider the sequence $\v_1, \ldots, \v_T \in \R^d$ such that $\|\v_i \| \leq B$ and define $V_{t} = \lambda I + \sum_{s \in [t - 1]} \v_s\v_s^\top$. Then,
\begin{align*}
    \sum_{t\in [T]} \| \v_t\|_{V_t^{-1}}^2 \leq 2 d \log \left( 1 + \frac{T B }{d } \right)
\end{align*}
\end{lemma}
\begin{proof}
Since $\lambda \geq B$, we have that $\|\v_{t}\|_{V_t^{-1}} \leq 1$. Therefore, from Lemma 19.4 of \cite{lattimore2020bandit}
\begin{align*}
    \sum_{t \in [T]} \|\v_{t}\|_{V_t^{-1}}^2  & \leq \sum_{t \in [T]} \log \left( 1 + \|\v_{t}\|_{V_t^{-1}}^2 \right) \\
    & \leq 2 d \log \left( \frac{d \lambda + T B^2 }{d \lambda} \right) \\
    & = 2 d \log \left( 1 + \frac{T B }{d } \right)
\end{align*}
\end{proof}

\regk*

\begin{proof}[Proof of Theorem~\ref{thm::regret-known}]
Armed with the supporting results, we now focus on completing the proof of Theorem~\ref{thm::regret-known}. The result may be shown by bounding the instantaneous regret. Condition on the event $\mathcal E_{\delta}$. Then,
\begin{align*}
    2r_t & := (\phi(\pi^*) - \phi(\pi_t^1))^\top \w^* + (\phi(\pi^*) - \phi(\pi_t^2))^\top \w^*\\
& = (\phi(\pi^*) - \phi(\pi_t^1))^\top \w^L_t + (\phi(\pi^*) - \phi(\pi_t^1))^\top (\w^* - \w_t^L) + (\phi(\pi^*) - \phi(\pi_t^2))^\top \w^L_t  \\
& \quad + (\phi(\pi^*) - \phi(\pi_t^2))^\top (\w^* - \w^L_t)    \\
& \leq (\phi(\pi^*) - \phi(\pi_t^1))^\top \w^L_t + (\phi(\pi^*) - \phi(\pi_t^2))^\top \w^L_t \\
& \quad + \|\w^* - \w^L_t \|_{\overline \V_t} \cdot \| \phi(\pi^*) - \phi(\pi_t^1) \|_{\overline \V_t^{-1}} + \|\w^* - \w^L_t \|_{\overline \V_t} \cdot \| \phi(\pi^*) - \phi(\pi_t^2) \|_{\overline \V_t^{-1}}
\end{align*}
The last two terms in the above sum can be bounded using Corollary~\ref{corollary::V_bar_bound} as follows:
\begin{align*}
    & \|\w^* - \w^L_t \|_{\overline \V_t} \cdot \| \phi(\pi^*) - \phi(\pi_t^1) \|_{\overline \V_t^{-1}} + \|\w^* - \w^L_t \|_{\overline \V_t} \cdot \| \phi(\pi^*) - \phi(\pi_t^2) \|_{\overline \V_t^{-1}} \\& \leq \left(2\kappa \beta_{t}(\delta) + \alpha_{T, d}(\delta) \right) \cdot \left(\| \phi(\pi^*) - \phi(\pi_t^1) \|_{\overline \V_t^{-1}} + \| \phi(\pi^*) - \phi(\pi_t^2) \|_{\overline \V_t^{-1}}\right)
\end{align*}
The first two terms leverage the optimistic bonus, using the fact that $\pi^1_t, \pi^2_t \in \mathcal S_t$: 
\begin{align*}
    (\phi(\pi^*) - \phi(\pi_t^1))^\top \w^L_t + (\phi(\pi^*) - \phi(\pi_t^2))^\top \w^L_t & \leq  \left(2\kappa \beta_{t}(\delta) + \alpha_{T, d}(\delta) \right) \cdot \left(\| \phi(\pi^*) - \phi(\pi_t^1) \|_{\overline \V_t^{-1}} + \| \phi(\pi^*) - \phi(\pi_t^2) \|_{\overline \V_t^{-1}}\right)
\end{align*}
In summary, we have that the instantaneous regret is upper bounded as
\begin{align*}
    2r_t & \leq 2 \left(2\kappa\beta_{t}(\delta) + \alpha_{T, d}(\delta) \right) \cdot \left(\| \phi(\pi^*) - \phi(\pi_t^1) \|_{\overline \V_t^{-1}} + \| \phi(\pi^*) - \phi(\pi_t^2) \|_{\overline \V_t^{-1}}\right) \\
    & \leq 4 \left(2\kappa \beta_{t}(\delta) + \alpha_{T, d}(\delta) \right) \cdot \|  \phi(\pi_t^1) - \phi(\pi_t^2)  \|_{\overline \V_t^{-1}} 
\end{align*}
where the last inequality follows from the fact that $\pi^* \in \mathcal S_t$ by Lemma~\ref{lem::contains-optimal} and since $\pi^1_t$ and $\pi^2_t$ were chosen the maximizer of the weighted difference $\|\phi(\pi_t^1) - \phi(\pi_t^2) \|_{\overline \V_{t}} $. The regret is therefore
\begin{align*}
    R_T & = \sum_{t \in [T]} r_t  \\
    &  \leq 2\left(2\kappa \beta_{T}(\delta) + \alpha_{T, d}(\delta) \right) \cdot \sum_{t \in [T]}  \|  \phi(\pi_t^1) - \phi(\pi_t^2)  \|_{\overline \V_t^{-1}} \\
    & \leq 2\left(2\kappa \beta_{T}(\delta) + \alpha_{T, d}(\delta) \right) \cdot \sqrt{ T \sum_{t \in [T]} \|  \phi(\pi_t^1) - \phi(\pi_t^2)  \|_{\overline \V_t^{-1}}^2  } \\
    & \leq 2\left(2\kappa \beta_{T}(\delta) + \alpha_{T, d}(\delta) \right) \cdot \sqrt{2  T  d \log \left( 1 + \frac{T B }{d } \right)  } 
\end{align*}
where the second inequality follows from Cauchy-Schwarz and the last inequality applies Lemma~\ref{lem::determinant}.
\end{proof}

\section{Appendix for Section~\ref{sec:ub}}\label{section::proofs_unknown_model_all}

In this section we will use the notation $N_t(s,a)$ to denote the number of times action $a$ was executed at state $s$ up to time $t-1$. Recall the bonus terms,

Given any $\eta > 0$ define,
\begin{align*}
\pushleft{\xi^{(t)}_{s,a}(\eta, \delta) = \min\left(2\eta, 4\eta \sqrt{\frac{U}{N_t(s,a)} }\right)  }  \\
\text{s.t. } U = H\log(|\mathcal{S}||\mathcal{A}|) + \log\left( \frac{6\log(N_t(s,a)}{\delta}\right).
\end{align*}
and the empirical average of $\xi^{(t)}_{s,a}(\eta, \delta) $ bonuses,
\begin{equation*}
    \widehat{B}_t(\pi, \eta, \delta) =  \mathbb{E}_{s_1 \sim \rho, \tau \sim \hat{\mathbb{P}}_t^{\pi}(\cdot | s_1)} \left[ \sum_{h=1}^{H-1} \xi_{s_h, a_h}^{(t)}(\eta, \delta)  \right].
\end{equation*}
Additionally we also define the error terms 
\begin{align*}
\xi^{(t)}_{s,a}(\epsilon, \eta, \delta) &=  \min\left(2\eta,4\eta \sqrt{ \frac{U}{N_t(s,a)}}\right)\\
\text{s.t. } U &= H\log(|\cS||\cA|H)+|\cS|\log\left(\ceil{\frac{4\eta H}{\epsilon}}\right)+ \log\left(\frac{6\log(N_t(s,a))}{\delta}\right).
\end{align*}
In contrast with the definition of bonus $\xi^{(t)}_{s,a}(\eta, \delta)$ this quantity depends on an extra parameter $\epsilon$. These erorr terms induce the the following `bonus' function,
\begin{equation*}
    B_t(\pi, \eta, \delta, \epsilon) = \mathbb{E}_{s_1 \sim \rho, \tau \sim \mathbb{P}^{\pi}(\cdot | s_1)} \left[ \sum_{h=1}^{H-1} \xi_{s_h, a_h}^{(t)}(\epsilon, \eta, \delta)  \right].
\end{equation*}

Here the expectation is under the true MDP dynamics.


Once we have established the validity of Lemma~\ref{lemma::regret_bound_unknown_support}, and therefore that with probability at least $1-15\delta$,
\begin{align*}
    R_T &\leq 2\gamma_T \sqrt{2Td\log\left( 1 + \frac{TB}{d}\right) } +  \sum_{t \in [T]} 4\widehat{B}_t(\pi_t^1, 4WB, \delta) + 4\widehat{B}_t(\pi_t^2, 4WB, \delta)
\end{align*}

it remains to show the $\widehat{B}_t( )$ terms are small. We'll do so by showing that for any $\eta > 0$ and $\delta \in (0,1)$ and for all policies $\pi$ simultaneously we can bound the empirical expected bonuses $\widehat{B}_t(\pi, \eta, \delta )$ in terms of the population quantities $B_t(\pi, \eta H, \delta)$,
\begin{lemma}\label{lemma::b_hat_to_bt}
Let $\eta, \epsilon > 0$. For all $\pi$ simultaneously and for all $t \in \mathbb{N}$, with probability $1-\delta$,
\begin{equation*}
    \widehat{B}_t(\pi, \eta, \delta) \leq 2B_t(\pi, 2H\eta , \delta, \epsilon) +\epsilon
\end{equation*}
\end{lemma}

\begin{proof}
Recall that,
\begin{equation*}
 \widehat{B}_t(\pi, \eta, \delta) =  \mathbb{E}_{s_1 \sim \rho, \tau \sim \hat{\mathbb{P}}_t^{\pi}(\cdot | s_1)} \left[ \sum_{h=1}^{H-1} \xi_{s_h, a_h}^{(t)}(\eta, \delta)  \right].
\end{equation*}
 Let $f : \Gamma \rightarrow \mathbb{R}$ be defined as,
 \begin{equation*}
     f(\tau ) =  \sum_{h=1}^{H-1} \xi_{s_h, a_h}^{(t)}(\eta). 
 \end{equation*}
 It is easy to see that $f(\tau) \in (0, 2\eta H] $ for all $\tau \in \Gamma$. Therefore, a direct application of Lemma~\ref{lemma::trajectory_score_concentration2} implies that with probability at least $1-\delta$ and simultaneously for all $\pi$, and $t \in \mathbb{N}$,
 \begin{equation*}
 \widehat{B}_t(\pi, \eta, \delta)  \leq \mathbb{E}_{s_1 \sim \rho, \tau \sim \mathbb{P}^{\pi}(\cdot | s_1)} \left[ \sum_{h=1}^{H-1} \xi_{s_h, a_h}^{(t)}(\eta, \delta)  \right] + B_t(\pi, 2H\eta, \delta, \epsilon) +\epsilon
 \end{equation*}

 Since $\xi^{(t)}_{s,a}(\epsilon, \eta, \delta) \geq  \xi_{s, a}^{(t)}(\eta, \delta)$ for all $\epsilon >0$, $s,a \in \mathcal{S} \times \mathcal{A}$ and $\xi^{(t)}_{s,a}(\epsilon, \eta, \delta)$ is monotonic in $\eta$ we conclude that,
 
 \begin{align*}
     \mathbb{E}_{s_1 \sim \rho, \tau \sim \mathbb{P}^{\pi}(\cdot | s_1)} \left[ \sum_{h=1}^{H-1} \xi_{s_h, a_h}^{(t)}(\eta, \delta)  \right] &\leq  \mathbb{E}_{s_1 \sim \rho, \tau \sim \mathbb{P}^{\pi}(\cdot | s_1)} \left[ \sum_{h=1}^{H-1} \xi_{s_h, a_h}^{(t)}(\epsilon, \eta, \delta)  \right]\\
     &\leq \mathbb{E}_{s_1 \sim \rho, \tau \sim \mathbb{P}^{\pi}(\cdot | s_1)} \left[ \sum_{h=1}^{H-1} \xi_{s_h, a_h}^{(t)}(\epsilon, 2H\eta, \delta)  \right]  \\
     &= B_t(\pi, 2H\eta, \delta , \epsilon)
 \end{align*}
 
 Combining these inequalities the result follows.
 
\end{proof}

Let $\bar{\mathcal{E}}_3$ such that for all $T \in \mathbb{N}$ be the event that,
\begin{align*}
\sum_{t \in [T]} 4\widehat{B}_t(\pi_t^1, 4SB, \delta) + 4\widehat{B}_t(\pi_t^2, 4SB, \delta) \leq \epsilon T +  \sum_{t \in [T]} 8{B}_t(\pi_t^1, 8HSB, \delta, \epsilon) + 8{B}_t(\pi_t^2, 8HSB, \delta, \epsilon) 
\end{align*}
Invoking Lemmas~\ref{lemma::regret_bound_unknown_support} and~\ref{lemma::b_hat_to_bt} we can show $\bar{\mathcal{E}}_3$  occurs with probability at least $1-2\delta$. Let's bound the sum $\sum_{t \in [T]} 8{B}_t(\pi_t^1, 8HSB, \delta, \epsilon) + 8{B}_t(\pi_t^2, 8HSB, \delta, \epsilon) $. Consider the martingale difference sequences  $\{ {B}_t(\pi_t^1, 8HSB, \delta, \epsilon)  -   \sum_{h=1}^{H-1} \xi_{s^1_{t,h}, a^1_{t,h}}^{(t)}(\epsilon, 8HSB, \delta) \}_{t=1}^\infty$    and $\{ {B}_t(\pi_t^2, 8HSB, \delta, \epsilon)  -   \sum_{h=1}^{H-1} \xi_{s^2_{t,h}, a^2_{t,h}}^{(t)}(\epsilon, 8HSB, \delta) \}_{t=1}^\infty$ each with norm upper bound $32H^2SB$. By an anytime Hoeffding inequality (see Lemma~\ref{lemma::matingale_concentration_anytime} ) (since $\xi_{ s,a}( \epsilon, \eta, \delta) \leq 2\eta$ and therefore $\sum_{h} \xi_{s_h, a_h}(\epsilon, \eta, \delta) \leq 2H\eta $) applied to  with probability at least $1-2\delta$ for all $T \in \mathbb{N}$ simultaneously
\begin{align*}
    \sum_{t \in [T]} 8{B}_t(\pi_t^1, 4SB, \delta, \epsilon) + 8{B}_t(\pi_t^2, 4SB, \delta, \epsilon) \leq  8\sum_{t\in[T]}  \Big(\sum_{h=1}^{H-1} \xi_{s^1_{t,h}, a^1_{t,h}}^{(t)}(\epsilon, 8HSB, \delta)    + 
    \sum_{h=1}^{H-1} \xi_{s^2_{t,h}, a^2_{t,h}}^{(t)}(\epsilon, 8HSB, \delta)  \Big)  + \mathbf{I}.
\end{align*}
Where $\mathbf{I} = 128HSB \sqrt{ T H\log\left(  \frac{6\log(TH)}{\delta}  \right) }$.
In order to bound the remaining empirical error terms, we to the following standard result,

\begin{lemma}\label{lemmma::standard_error_bound}
For $i \in \{1,2\}$ the empirical sum of errors satisfies the following bound  
\begin{align*}
\pushleft{    \sum_{t\in[T]} \sum_{h=1}^{H-1} \xi_{s^i_{t,h}, a^i_{t,h}}^{(t)}(\epsilon, 8HSB, \delta)       \leq }\\
    64 HSB\sqrt{\left(H\log(|\cS||\cA|H)+|\cS|\log\left(\ceil{\frac{32H^2SB}{\epsilon}}\right)+\log\left(\frac{6\log(HTs)}{\delta}\right)\right) |\mathcal{S}||\mathcal{A}| TH }.
\end{align*}
\end{lemma}
\begin{proof}
Let's rewrite this sum by instead summing over states and actions,

\begin{align*}
\pushleft{\sum_{t\in[T]} \sum_{h=1}^{H-1} \xi_{s^i_{t,h}, a^i_{t,h}}^{(t)}(\epsilon, 8HSB, \delta)     =}\\
    \sum_{ s \in \mathcal{S}} \sum_{a \in \mathcal{A}} \sum_{t=1}^{N_{T+1}(s,a)} \min\left(16HSB, 32HSB \sqrt{ \frac{H\log(|\cS||\cA|H)+|\cS|\log\left(\ceil{\frac{32H^2SB}{\epsilon}}\right)+\log\left(\frac{6\log(t)}{\delta}\right)}{t}}\right)\\
    \leq   32HSB\sqrt{ H\log(|\cS||\cA|H)+|\cS|\log\left(\ceil{\frac{32H^2SB}{\epsilon}}\right)+\log\left(\frac{6\log(HT)}{\delta}\right) }  \sum_{ s \in \mathcal{S}} \sum_{a \in \mathcal{A}} \sum_{t=1}^{N_{T+1}(s,a)} \frac{1}{\sqrt{t}} \\
    \leq  32HSB\sqrt{ H\log(|\cS||\cA|H)+|\cS|\log\left(\ceil{\frac{32H^2SB}{\epsilon}}\right)+\log\left(\frac{6\log(HT)}{\delta}\right) } \sum_{ s \in \mathcal{S}} \sum_{a \in \mathcal{A}} 2\sqrt{N_{T+1}(s,a)} \\
    \leq 64 HSB\sqrt{\left(H\log(|\cS||\cA|H)+|\cS|\log\left(\ceil{\frac{32H^2SB}{\epsilon}}\right)+\log\left(\frac{6\log(HT)}{\delta}\right)\right) |\mathcal{S}||\mathcal{A}| TH }.
\end{align*}

The result follows.
\end{proof}

We will use Lemma~\ref{lemmma::standard_error_bound} with $\epsilon = 1/T$. As a consequence of Lemma~\ref{lemmma::standard_error_bound} and Lemma~\ref{lemma::regret_bound_unknown_support} we see that when $\mathcal{E}_{\delta} \cap \mathcal{E}_2 \cap \bar{\mathcal{E}}_0 \cap \bar{\mathcal{E}}_{-1} \cap \bar{\mathcal{E}}_{2} \cap \bar{\mathcal{E}}_3$ holds 

\begin{align*}
    R_T &\leq 2\gamma_T \sqrt{2Td\log\left( 1 + \frac{TB}{d}\right) }+ \widetilde{\mathcal{O}}\left( H^{3/2} \sqrt{|\mathcal{A}||\mathcal{S}|TH} + H|\mathcal{S}|\sqrt{|\mathcal{A}|TH } + H\sqrt{TH} \right).
\end{align*}
Now it remains to bound term $\gamma_T$. 
From now on let's set $\epsilon =\min(1/T, 8HSB)$ and let $\bar{\mathcal{E}}_4$ be the event that for all $t \in \mathbb{N}$ and all $i \in \{1,2\}$,
$$\widehat{B}_t\left(\pi^i_\ell, 2SB, \frac{\delta'}{8\ell^3\mathcal{A}^{\mathcal{S}} }\right) \leq 2B_t\left(\pi_\ell^i, 4HSB, \frac{\delta'}{8\ell^3\mathcal{A}^{\mathcal{S}} }, \epsilon\right) +\epsilon $$ for all $t \in \mathbb{N}$ and all $i \in \{1,2\}$. As a consequence of Lemma~\ref{lemma::b_hat_to_bt} we can bound $\mathbb{P}\left(\bar{\mathcal{E}}_4\right) \geq 1-2\delta$. Squaring both sides,
\begin{small}
\begin{align*}
    \left(\widehat{B}_t(\pi^i_\ell, 2SB, \delta'_\ell) \right)^2 &\leq \left(2B_t(\pi_\ell^i, 4HSB, \delta'_\ell, \epsilon) +\epsilon\right)^2 \\
    &\stackrel{(i)}{\leq} 4 \left(B_t(\pi_\ell^i, 4HSB, \delta'_\ell, \epsilon) \right)^2 + 16\epsilon HSB + \epsilon^2 \\
    &\leq 4 \left(B_t(\pi_\ell^i, 4HSB, \delta'_\ell, \epsilon) \right)^2 + 24\epsilon HSB 
\end{align*}
\end{small}
Where $\delta'_\ell = \frac{\delta'}{8\ell^3\mathcal{A}^{\mathcal{S}} }$. Inequality $(i)$ used that $B(\pi, \eta, \delta, \epsilon) \leq 2H\eta$. The last inequality holds because $\epsilon \leq 8HSB$. Therefore if $\bar{\mathcal{E}}_4$ holds for all $t \in \mathbb{N}$,
\begin{align*}
\gamma_t \leq \sqrt{2} \left( 4\kappa \beta_t(\delta) + \alpha_{d, T}(\delta)  \right) + \frac{1}{t} + 4\sqrt{\sum_{\ell=1}^{t-1} {B}^2_t\left(\pi_\ell^1, 4HSB, \delta'_\ell, \epsilon \right) + B^2_t\left(\pi_\ell^2, 4HSB, \delta'_\ell\right)  + \mathbf{I}  } . 
\end{align*}
Where $\mathbf{I} = 96(t-1)\epsilon HSB$. We are just left with bounding the sum of squares $\sum_{\ell=1}^{t-1} \left({B}_t\left(\pi_\ell^1, 4HSB, \frac{\delta'}{8\ell^3\mathcal{A}^{\mathcal{S}} }, \epsilon \right)\right)^2 + \left(B_t\left(\pi_\ell^2, 4HSB, \frac{\delta'}{8\ell^3\mathcal{A}^{\mathcal{S}} }\right)\right)^2$.
 
\begin{lemma}\label{lemma:squared_b_l_squared_bound}
Let $\eta , \epsilon >0$ and $\delta,\delta' \in (0,1)$ and define $\bar{\mathcal{E}}_5(\delta')$ be the event that for all $t \in \mathbb{N}$ and $i \in \{1,2\}$
\begin{align*}
\pushleft{\sum_{i \in \{1,2\}} \sum_{\ell=1}^{t-1} \left(B_\ell(\pi_\ell^{i}, \eta, \delta/\ell^3) \right)^2\leq 12\eta^2H^2 \left(  1.4 \ln \ln \left(2 \left(\max\left(4\eta^2 H t , 1\right)\right)\right) + \ln \frac{5.2}{\delta'} + 1\right) +}\\
\qquad 64\eta^2H \left(H\log(|\cS||\cA|H)+|\cS|\log\left(\ceil{\frac{4\eta H}{\epsilon}}\right)+\log\left(\frac{6\log(HT)}{\delta}\right) \right) |\mathcal{S}||\mathcal{A}| \log(TH + |\mathcal{S}||\mathcal{A}| )  
\end{align*}
 Then $\mathbb{P}( \bar{\mathcal{E}}_5(\delta')) \geq 1-2\delta'$.
\end{lemma}

\begin{proof}
Observe that,
\begin{align*}
   \left( B_\ell(\pi_t^1, \eta, \frac{\delta}{\ell^3}, \epsilon)  \right)^2+  \left( B_\ell(\pi_t^1, \eta, \frac{ \delta }{\ell^3}, \epsilon) \right)^2 &= \left(\mathbb{E}_{s^1_1 \sim \rho, \tau \sim \mathbb{P}^{\pi_t^1}(\cdot | s^1_1)} \left[ \sum_{h=1}^{H-1} \xi_{s^1_h, a^1_h}^{(t)}(\epsilon, \eta, \frac{\delta}{\ell^3})  \right]\right)^2 + \\
   &\quad \left( \mathbb{E}_{s^2_1 \sim \rho, \tau \sim \mathbb{P}^{\pi_t^2}(\cdot | s^2_1)} \left[ \sum_{h=1}^{H-1} \xi_{s^2_h, a^2_h}^{(t)}(\epsilon, \eta, \frac{\delta}{\ell^3})  \right] \right)^2\\
    &\stackrel{(i)}{\leq} H \mathbb{E}_{s^1_1 \sim \rho, \tau \sim \mathbb{P}^{\pi_t^1}(\cdot | s^1_1)} \left[ \sum_{h=1}^{H-1} \left(\xi_{s^1_h, a^1_h}^{(t)}(\epsilon, \eta, \frac{\delta}{\ell^3}) \right)^2 \right] +\\
    &\quad H\mathbb{E}_{s^2_1 \sim \rho, \tau \sim \mathbb{P}^{\pi_t^2}(\cdot | s^2_1)} \left[ \sum_{h=1}^{H-1} \left(\xi_{s^2_h, a^2_h}^{(t)}(\epsilon, \eta, \frac{\delta}{\ell^3})  \right)^2\right]
\end{align*}
Where inequality $(i)$ is a consequence of $\left(\mathbb{E}\left[ \sum_{h=1}^H a_h \right] \right)^2 \leq H\mathbb{E}\left[ \sum_{h=1}^H a_h^2 \right]  $. Define the martingale-difference sequences for $i \in \{ 1,2\}$,
\begin{align*}
    D_\ell^{(i)} &=  \mathbb{E}_{s^i_1 \sim \rho, \tau \sim \mathbb{P}^{\pi_\ell^i}(\cdot | s^i_1)} \left[ \sum_{h=1}^{H-1} \left(\xi_{s^i_h, a^i_h}^{(\ell)}(\epsilon, \eta, \delta) \right)^2 \right]  - \sum_{h=1}^{H-1} \left(\xi_{s^i_h, a^i_h}^{(\ell)}(\epsilon, \eta, \delta) \right)^2
\end{align*}

Since $\xi^{(\ell)}_{s,a}(\epsilon, \eta, \delta) \leq 2\eta$, we see that $ \left| D_{\ell}^{(i)} \right| \leq  8\eta^2 H$. Observe that for $i \in \{1,2\}$,

\begin{align*}
    \mathrm{Var}^{(i)}_\ell\left(  \sum_{h=1}^{H-1} \left(\xi_{s^i_h, a^i_h}^{(\ell)}(\epsilon, \eta, \delta) \right)^2  \right) &\leq  \mathbb{E}_{s^i_1 \sim \rho, \tau \sim \mathbb{P}^{\pi_\ell^i}(\cdot | s^i_1)} \left[\left\{ \sum_{h=1}^{H-1} \left(\xi_{s^i_h, a^i_h}^{(\ell)}(\epsilon, \eta, \delta) \right)^2\right\}^2 \right] \\
    &\stackrel{(i)}{\leq} 4\eta^2 H  \mathbb{E}_{s^i_1 \sim \rho, \tau \sim \mathbb{P}^{\pi_\ell^i}(\cdot | s^i_1)} \left[ \sum_{h=1}^{H-1} \left(\xi_{s^i_h, a^i_h}^{(\ell)}(\epsilon, \eta, \delta) \right)^2 \right] \\
   &\leq 16\eta^4 H^2 
\end{align*}

Where $(i)$ follows follows because for $\xi^{(\ell)}_{s,a}(\epsilon, \eta, \delta) \leq 2\eta$.

Since the variance can be bounded by the mean, we can make use of a Uniform Empirical Bernstein Bound from Lemma~\ref{lem:uniform_emp_bernstein}. Let $S^{(i)}_t= \sum_{\ell=1}^t D_\ell^{(i)}$ for $i \in \{1,2\}$ and $W^{(i)}_t = \sum_{\ell=1}^t    \mathrm{Var}^{(i)}_\ell\left(  \sum_{h=1}^{H-1} \left(\xi_{s_h, a_h}^{(\ell,i)}(\epsilon, \eta, \delta) \right)^2  \right)$. Let $c = 8\eta^2H $ and $m = 4\eta^2H$. With probability $1-\delta'$ for all $t \in \mathbb{N}$,


\begin{align*}
    \sum_{\ell=1}^{t-1} D_\ell^{(i)} &\leq  \sqrt{\max\left( 4\eta^2H \sum_{\ell = 1}^{t-1}   \mathbb{E}_{s_1 \sim \rho, \tau \sim \mathbb{P}^{\pi_\ell^i}(\cdot | s_1)} \left[ \sum_{h=1}^{H-1} \left(\xi_{s_h, a_h}^{(\ell, i)}(\epsilon, \eta, \delta) \right)^2 \right], 4\eta^2 H\right) \left(  1.4 \ln \ln \left(2 \left(\max\left(4\eta^2 H t , 1 \right)\right)\right) + \ln \frac{5.2}{\delta'} \right)  } \\
    &+  3.28 \eta^2 H\eta   \left( 1.4 \ln \ln \left(2 \left(\max\left(4\eta^2 H t , 1\right)\right)\right) + \ln \frac{5.2}{\delta'}\right) \\
    &\leq \sqrt{\left( 4\eta^2H \sum_{\ell = 1}^{t-1}   \mathbb{E}_{s_1 \sim \rho, \tau \sim \mathbb{P}^{\pi_\ell^i}(\cdot | s_1)} \left[ \sum_{h=1}^{H-1} \left(\xi_{s_h, a_h}^{(\ell, i)}(\epsilon, \eta, \delta) \right)^2 \right] + 4\eta^2H \right) \left(  1.4 \ln \ln \left(2 \left(\max\left(4\eta^2 H t , 1 \right)\right)\right) + \ln \frac{5.2}{\delta'} \right)  } \\
    &+  3.28 \eta^2 H   \left( 1.4 \ln \ln \left(2 \left(\max\left(4\eta^2 H t , 1\right)\right)\right) + \ln \frac{5.2}{\delta'}\right)
\end{align*}

Since $\sqrt{ab} \leq \frac{ a + b }{2}$,

\begin{align*}
      \sum_{\ell=1}^{t-1} D_\ell^{(i)} &\leq \frac{1}{2}\mathbb{E}_{s_1 \sim \rho, \tau \sim \mathbb{P}^{\pi_\ell^i}(\cdot | s_1)} \left[ \sum_{h=1}^{H-1} \left(\xi_{s_h, a_h}^{(\ell, i)}(\epsilon, \eta, \delta) \right)^2 \right] + \\
      &\quad 2\eta^2 H + \underbrace{(3.28 \eta^2 H+ 2\eta^2 H)}_{\leq 6 \eta^2 H} \left(  1.4 \ln \ln \left(2 \left(\max\left(4\eta^2 H t , 1\right)\right)\right) + \ln \frac{5.2}{\delta'}\right). 
\end{align*}

Therefore with high probability for $i \in \{1, 2\}$,
\begin{align*}
\mathbb{E}_{s_1 \sim \rho, \tau \sim \mathbb{P}^{\pi_\ell^i}(\cdot | s_1)} \left[ \sum_{h=1}^{H-1} \left(\xi_{s_h, a_h}^{(\ell, i)}(\epsilon, \eta, \delta) \right)^2 \right]
&\leq 2\sum_{\ell=1}^{t-1}\sum_{h=1}^{H-1} \left(\xi_{s_h, a_h}^{(\ell,i)}(\epsilon, \eta, \delta) \right)^2 + 4\eta^2 H + \\
&\quad 6\eta^2H \left(  1.4 \ln \ln \left(2 \left(\max\left(4\eta^2 H t , 1\right)\right)\right) + \ln \frac{5.2}{\delta'}\right). 
\end{align*}

Therefore with probability $1-2\delta'$,
\begin{align*}
\sum_{i \in \{1,2\}} \sum_{\ell=1}^{t-1} \left(B_\ell(\pi_\ell^{i}, \eta, \delta/\ell^3) \right)^2&\leq 2H\sum_{i \in \{1,2\}}\sum_{\ell=1}^{t-1}\sum_{h=1}^{H-1} \left(\xi_{s_h, a_h}^{(\ell,i)}(\epsilon, \eta, \delta) \right)^2  +  \\
&\quad 12\eta^2H^2 \left(  1.4 \ln \ln \left(2 \left(\max\left(4\eta^2 H t , 1\right)\right)\right) + \ln \frac{5.2}{\delta'} + 1\right). 
\end{align*}

We are left with the task of bounding the terms $\sum_{\ell=1}^{t-1}   \sum_{h=1}^{H-1} \left(\xi_{s_h, a_h}^{(\ell,i)}(\epsilon, \eta, \delta) \right)^2$.

Let's rewrite this sum by instead summing over states and actions,

\begin{align*}
  \pushleft{  \sum_{t\in[T]} \sum_{h=1}^{H-1} \left(\xi_{s^i_{t,h}, a^i_{t,h}}^{(t)}(\epsilon, \eta, \delta)   \right)^2  =}\\
    \sum_{ s \in \mathcal{S}} \sum_{a \in \mathcal{A}} \sum_{t=1}^{N_{T+1}(s,a)} \min\left(4\eta^2, 16\eta^2  \frac{H\log(|\cS||\cA|H)+|\cS|\log\left(\ceil{\frac{4\eta H}{\epsilon}}\right)+\log\left(\frac{6\log(t)}{\delta}\right)}{t}\right)\\
    =\sum_{ s \in \mathcal{S}} \sum_{a \in \mathcal{A}} \sum_{t=1}^{N_{T+1}(s,a)}  16\eta^2  \frac{H\log(|\cS||\cA|H)+|\cS|\log\left(\ceil{\frac{4\eta H}{\epsilon}}\right)+\log\left(\frac{6\log(t)}{\delta}\right)}{t}\\
   = 16\eta^2 \left(H\log(|\cS||\cA|H)+|\cS|\log\left(\ceil{\frac{4\eta H}{\epsilon}}\right)+\log\left(\frac{6\log(HT)}{\delta'}\right) \right) \sum_{ s \in \mathcal{S}} \sum_{a \in \mathcal{A}} \sum_{t=1}^{N_{T+1}(s,a)}   \frac{1}{t}\\
   32\eta^2 \left(H\log(|\cS||\cA|H)+|\cS|\log\left(\ceil{\frac{4\eta H}{\epsilon}}\right)+\log\left(\frac{6\log(HT)}{\delta'}\right) \right) \sum_{ s \in \mathcal{S}} \sum_{a \in \mathcal{A}} \log\left(  N_{T+1}(s,a) +1 \right)\\
   \leq 32\eta^2 \left(H\log(|\cS||\cA|H)+|\cS|\log\left(\ceil{\frac{4\eta H}{\epsilon}}\right)+\log\left(\frac{6\log(HT)}{\delta'}\right) \right) |\mathcal{S}||\mathcal{A}| \log(TH + |\mathcal{S}||\mathcal{A}| )
\end{align*}

Therefore with probability $1-2\delta'$,

\begin{align*}
\pushleft{\sum_{i \in \{1,2\}} \sum_{\ell=1}^{t-1} \left(B_\ell(\pi_\ell^{i}, \eta, \delta/\ell^3) \right)^2\leq 12\eta^2H^2 \left(  1.4 \ln \ln \left(2 \left(\max\left(4\eta^2 H t , 1\right)\right)\right) + \ln \frac{5.2}{\delta'} + 1\right) +}\\
\qquad 64\eta^2H \left(H\log(|\cS||\cA|H)+|\cS|\log\left(\ceil{\frac{4\eta H}{\epsilon}}\right)+\log\left(\frac{6\log(HT)}{\delta'}\right) \right) |\mathcal{S}||\mathcal{A}| \log(TH + |\mathcal{S}||\mathcal{A}| )  
\end{align*}
The result follows. \end{proof}

The main takeaway from this lemma is that the sum of the square errors grows only logarithmically in $T$. Applying this bound to $\gamma_T$ and setting $\epsilon = O(1/T)$ we obtain,
\begin{align*}
  \gamma_T &\leq \sqrt{2}\left( 4\kappa \beta_T(\delta) + \alpha_{d, T}(\delta)  \right) + 2\sqrt{  \omega_T(\delta)    } +\frac{1}{T} = \widetilde{\mathcal{O}}\left( \kappa \sqrt{d} +  H^2 \sqrt{|\mathcal{S}||\mathcal{A}|} + H^{3/2}|\mathcal{S}|\sqrt{ |\mathcal{A}|} \right)
\end{align*}

Where 
\begin{align*}
    \omega_T(\delta) &= 192 H^4 S62B^2 \left(  1.4 \ln \ln \left(2 \left(\max\left(64 H^3 S^2B^2 t , 1\right)\right)\right) + \ln \frac{5.2}{\delta'} + 1\right) + \\
    &\qquad 1024H^3S^2B^2 \left(H\log(|\cS||\cA|H)+|\cS|\log\left(\ceil{\frac{64 H^3 S^2B^2}{\epsilon}}\right)+\log\left(\frac{6\log(HT)}{\delta'}\right) \right) |\mathcal{S}||\mathcal{A}| \log(TH + |\mathcal{S}||\mathcal{A}| )  
\end{align*}
Applying this bound to $\gamma_T$ and setting $\epsilon = 1/96T$ we obtain,
\begin{align*}
  \gamma_T \leq \sqrt{2}\left( 4\kappa \beta_T(\delta) + \alpha_{d, T}(\delta)  \right) + 2\sqrt{  \omega_T(\delta) +  HSB  } +\frac{1}{T} 
\end{align*}

Combining these observations we can derive our main result,

\begin{thm}[Formal version of Theorem~\ref{theorem:main_unknown_model_simple}]\label{theorem:main_unknown_model_appendix}
If $\mathcal{E}_{\delta} \cap \mathcal{E}_2 \cap \bar{\mathcal{E}}_0 \cap \bar{\mathcal{E}}_{-1} \cap \bar{\mathcal{E}}_{2} \cap \bar{\mathcal{E}}_3 \cap \bar{\mathcal{E}}_4 \cap \bar{\mathcal{E}}_5(\frac{\delta}{2})$ holds then the regret of $\mathbf{LPbRL}$ satisfies, 
\begin{align*}
    R_T &\leq 2\left( 4\kappa \beta_t(\delta) + \alpha_{d, T}(\delta)   + 2\sqrt{  \omega_T(\delta) + 96(t-1)\epsilon HSB  } +\frac{1}{T}\right) \sqrt{2Td\log\left( 1 + \frac{TB}{d}\right) }+\\
    &\quad 128H^{3/2} SB \sqrt{ T H\log\left(  \frac{6\log(T)}{\delta}  \right) } + \\
    &\quad  1024 HSB\sqrt{\left(H\log(|\cS||\cA|H)+|\cS|\log\left(\ceil{\frac{32H^2SB}{\epsilon}}\right)+\log\left(\frac{6\log(HT)}{\delta}\right)\right) SA TH }.
\end{align*}
For all $T \in \mathbb{N}$ simultaneously. Where $\mathbb{P}\left( \mathcal{E}_{\delta} \cap \mathcal{E}_2 \cap \bar{\mathcal{E}}_0 \cap \bar{\mathcal{E}}_{-1} \cap \bar{\mathcal{E}}_{2} \cap \bar{\mathcal{E}}_3 \cap \bar{\mathcal{E}}_4 \cap \bar{\mathcal{E}}_5(\frac{\delta}{2})\right) \geq 1-15\delta$.
\end{thm}



\subsection{Supporting Related Work Lemmas}


We will make use of the following Lemma (see Lemma B.1 in~\cite{chatterji21}),

\begin{lemma}\label{lemma::trajectory_score_concentration1}
For any \underline{fixed} policy $\pi $, and any function $f: \Gamma \rightarrow \mathbb{R}$ satisfying $\max_{\tau \in \Gamma} |f(\tau)| \leq \eta$, with probability at least $1-\delta$ for all $t \in \mathbb{N}$,
\begin{equation*}
    \mathbb{E}_{s_1 \sim \rho, \tau \sim \mathbb{P}^{\pi}(\cdot | s_1)} \left[ f(\tau)  \right]  -  \mathbb{E}_{s_1 \sim \rho, \tau \sim \hat{\mathbb{P}}_t^{\pi}(\cdot | s_1)} \left[ f(\tau)  \right] \leq    \widehat{B}_t(\pi, \eta, \delta)
\end{equation*}
\end{lemma}

We will also make use of the following Lemma (see Lemma B.2 from~\cite{chatterji21} ) corresponding to the uniform version of lemma~\ref{lemma::trajectory_score_concentration1}.

\begin{lemma}[Uniform version of Lemma~\ref{lemma::trajectory_score_concentration1} ] \label{lemma::trajectory_score_concentration2}
Let $\epsilon >0$. For any function $f: \Gamma \rightarrow \mathbb{R}$ satisfying $\max_{\tau \in \Gamma} |f(\tau)| \leq \eta$, for all policies $\pi$ simultaneously and all $t \in \mathbb{N}$,
\begin{equation*}
    \mathbb{E}_{s_1 \sim \rho, \tau \sim \widehat{\mathbb{P}}^{\pi}(\cdot | s_1)} \left[ f(\tau)  \right]  -  \mathbb{E}_{s_1 \sim \rho, \tau \sim \mathbb{P}_t^{\pi}(\cdot | s_1)} \left[ f(\tau)  \right] \leq  B_t(\pi, \eta, \delta, \epsilon) + \epsilon.
\end{equation*}
\end{lemma}

We will make use of the following standard bound on the covering number of the $l_2$ ball.

\begin{lemma}\label{lemma::covering_number_unit_ball}
For any $\epsilon \in(0,1]$ the $\epsilon-$covering number of the Euclidean ball in $\mathbb{R}^d$ with radius $r > 0$ i.e.. $\{ \mathbf{x} \in \mathbb{R}^d : \| \mathbf{x} \|_2 \leq r \} $ is upper bounded by $\left(\frac{1+2r}{\epsilon}\right)^d$.
\end{lemma}

\subsection{Proof of Lemma~\ref{lemma::upper_bound_upper_bounding_w_L_norm_unknown1}}\label{section::proof_upper_bound_upper_bounding_w_L_norm_unknown1}

\lemmaupperboundupperboundingwlnormunknownone*

\begin{proof}

Recall that as a result of assumption~\ref{assumption::bounded_parameter} and the definition of $\w_t^L$ we can bound $\| \w_t^L - \w^*\| \leq 2S$. Let $\v$ be such that $\|\v \| \leq 2S$.

Let's consider $\v^\top \widetilde{\mathbf{V}}_t \v $,

\begin{align*}
    \v^\top \widetilde{\mathbf{V}}_t \v  &= \kappa \lambda \| \v \|^2 + \sum_{\ell=1}^{t-1} \left( \left\langle \v , \phi^{\widehat{\mathbb{P}}_\ell}(\pi_\ell^1) - \phi^{\widehat{\mathbb{P}}_\ell}(\pi_\ell^2)  \right\rangle \right)^2  
\end{align*}

Let's focus on a single summand $\left\langle \v , \phi^{\widehat{\mathbb{P}}_\ell}(\pi_\ell^1) - \phi^{\widehat{\mathbb{P}}_\ell}(\pi_\ell^2)  \right\rangle^2  $ with $\ell \in [t-1]$.

By Lemma~\ref{lemma::trajectory_score_concentration1}, with probability at least $1-\frac{\delta'}{4\ell^3}$ for all $\pi \in \Pi$ simultaneously,

\begin{equation*}
    \langle \phi^{\widehat{\mathbb{P}}_\ell}(\pi), \v \rangle \leq  \langle \phi(\pi), \v \rangle + \widehat{B}_\ell\left(\pi, 2SB, \frac{\delta'}{8\ell^3\mathcal{A}^{\mathcal{S}} }\right) .
\end{equation*}
And 
\begin{equation*}
    -\langle \phi^{\widehat{\mathbb{P}}_\ell}(\pi), \v \rangle \leq  -\langle \phi(\pi), \v \rangle + \widehat{B}_\ell\left(\pi, 2SB, \frac{\delta'}{8\ell^3\mathcal{A}^{\mathcal{S}} }\right) .
\end{equation*}

Then

\begin{equation*}
    \left|\langle \phi^{\widehat{\mathbb{P}}_\ell}(\pi_\ell^1) - \phi^{\widehat{\mathbb{P}}_\ell}(\pi_\ell^2), \v \rangle \right| \leq  \left| \langle \phi(\pi_\ell^1) - \phi(\pi_\ell^2), \v \rangle \right| + \widehat{B}_\ell\left(\pi_\ell^1, 2SB, \frac{\delta'}{8\ell^3\mathcal{A}^{\mathcal{S}} }\right) + \widehat{B}_\ell\left(\pi_\ell^2, 2SB, \frac{\delta'}{8\ell^3\mathcal{A}^{\mathcal{S}} }\right) .
\end{equation*}

And therefore,

\begin{equation*}
    \langle \phi^{\widehat{\mathbb{P}}_\ell}(\pi_\ell^1) -\phi^{\widehat{\mathbb{P}}_\ell}(\pi_\ell^2) , \v \rangle^2  \leq   2\langle \phi(\pi_\ell^1) - \phi(\pi_\ell^2), \v \rangle^2 + 4\left(\widehat{B}_\ell\left(\pi_\ell^1, 2SB, \frac{\delta'}{8\ell^3\mathcal{A}^{\mathcal{S}} }\right)\right)^2 + 4\left(\widehat{B}_\ell\left(\pi_\ell^2, 2SB, \frac{\delta'}{8\ell^3\mathcal{A}^{\mathcal{S}} }\right)\right)^2 .
\end{equation*}

Therefore with probability at least $1-\frac{\delta}{2t^2}$, for all $t \in \mathbb{N}$ simultaneously,

\begin{align*}
    \v^\top \widetilde{\mathbf{V}}_t \v  &= \kappa \lambda \| \v \|^2 + \sum_{\ell=1}^{t-1}  \left\langle \v , \phi^{\widehat{\mathbb{P}}_\ell}(\pi_\ell^1) - \phi^{\widehat{\mathbb{P}}_\ell}(\pi_\ell^2)  \right\rangle^2  \\
    &\leq  2\kappa \lambda \| \v \|^2 + 2\sum_{\ell=1}^{t-1}  \left\langle \v , \phi(\pi_\ell^1) - \phi(\pi_\ell^2)  \right\rangle^2 + 4\left(\widehat{B}_t\left(\pi_\ell^1, 2SB, \frac{\delta'}{8\ell^3\mathcal{A}^{\mathcal{S}} }\right)\right)^2 +4\left(\widehat{B}_t\left(\pi_\ell^2, 2SB, \frac{\delta'}{8\ell^3\mathcal{A}^{\mathcal{S}} }\right)\right)^2 \\
    &= 2\v^\top \bar{\mathbf{V}}_t \v + \sum_{\ell=1}^{t-1} 4\left(\widehat{B}_t\left(\pi_\ell^1, 2SB, \frac{\delta'}{8\ell^3\mathcal{A}^{\mathcal{S}} }\right)\right)^2 + 4\left(\widehat{B}_t\left(\pi_\ell^2, 2SB, \frac{\delta'}{8\ell^3\mathcal{A}^{\mathcal{S}} }\right)\right)^2
\end{align*}

Consider an $\epsilon-$cover of the $2S$ ball in $\mathbb{R}^d$ and let's see that for any other $\v $ in the $2S$ ball, the closest vector in the covering $\tilde{\v}$ satisfies, $\| \v - \tilde{\v} \| \leq \epsilon$ and therefore,

\begin{align*}
    \left| \v^\top \widetilde{\mathbf{V}}_t \v  - \tilde{\v}^\top \widetilde{\mathbf{V}}_t \tilde{\v}  \right| &=   \left| \v^\top \widetilde{\mathbf{V}}_t \v - \v^\top \widetilde{\mathbf{V}}_t \tilde{\v}  + \v^\top \widetilde{\mathbf{V}}_t \tilde{\v}   - \tilde{\v}^\top \widetilde{\mathbf{V}}_t \tilde{\v}  \right| \\
    &\leq \left| \v^\top \widetilde{\mathbf{V}}_t \v - \v^\top \widetilde{\mathbf{V}}_t \tilde{\v} \right| + \left|\v^\top \widetilde{\mathbf{V}}_t \tilde{\v}   - \tilde{\v}^\top \widetilde{\mathbf{V}}_t \tilde{\v}  \right|\\
    &\leq \|  \widetilde{\mathbf{V}}_t \v \| \| \v - \tilde{\v}\| + \|  \widetilde{\mathbf{V}}_t \tilde{\v} \| \| \v - \tilde{\v}\| \\
    &\leq \epsilon\left(  \kappa \lambda + 4B^2(t-1)\right)
\end{align*}

Similarly
\begin{equation*}
      \left| \v^\top \bar{\mathbf{V}}_t \v  - \tilde{\v}^\top \bar{\mathbf{V}}_t \tilde{\v}  \right|  \leq \epsilon\left(   \kappa \lambda + 4B^2(t-1) \right) 
\end{equation*}
Invoking Lemma~\ref{lemma::covering_number_unit_ball} and setting $\delta'= \frac{\delta}{\left( \frac{1+4S}{\epsilon}\right)^d}$ applying it to $\w_t^L - \w_* $
\begin{align*}
    \| \w_t^L - \w_* \|^2_{\widetilde{\rmV}_t}  &\leq 2\| \w_t^L - \w_* \|^2_{\bar{\rmV}_t} +  \\
    &\quad \sum_{\ell=1}^{t-1} 4\left(\widehat{B}_t\left(\pi_\ell^1, 2SB, \frac{\delta'}{8\ell^3\mathcal{A}^{\mathcal{S}} }\right)\right)^2 + 4\left(\widehat{B}_t\left(\pi_\ell^2, 2SB, \frac{\delta'}{8\ell^3\mathcal{A}^{\mathcal{S}} }\right)\right)^2 + 2\epsilon\left(   \kappa \lambda + 4B^2(t-1) \right) .
\end{align*}

Setting $\epsilon = \frac{1}{t^2\kappa \lambda + 4B^2t^3 }$ and using the fact that all $a,b,c\geq 0$ we have $\sqrt{a^2+b^2 + c^2} \leq a+b+c$, 

\begin{equation*}
        \| \w_t^L - \w_* \|_{\widetilde{\rmV}_t}  \leq \sqrt{2}\| \w_t^L - \w_* \|_{\bar{\rmV}_t} +  \sqrt{\sum_{\ell=1}^{t-1} 4\left(\widehat{B}_t\left(\pi_\ell^1, 2SB, \frac{\delta'}{8\ell^3\mathcal{A}^{\mathcal{S}} }\right)\right)^2 + 4\left(\widehat{B}_t\left(\pi_\ell^2, 2SB, \frac{\delta'}{8\ell^3\mathcal{A}^{\mathcal{S}} }\right)\right)^2  } + \frac{1}{t}.
\end{equation*}

Since the inequality holds for any $t \in \mathbb{N}$ with probability at least $1-\frac{\delta}{2t^2}$, by the union bound, the inequality holds for all $t \in \mathbb{N}$ simultaneously with probability at least $1-\delta$.

\end{proof}

\subsection{Proof of Lemma~\ref{lem::contains-optimal_unknown}}\label{section:proof_lemma_pi_star_containment}

\lemcontainsoptimalunknown*

\begin{proof}

Let's start by conditioning on $\mathcal{E}_\delta$, $\bar{\mathcal{E}}_{0}$ and $\mathcal E_2$ (see Corollary~\ref{corollary::V_bar_bound} for a definition of $\mathcal{E}_2$). By Lemmas~\ref{lemma::confidence_interval_anytime} and~\ref{lemma::upper_bound_upper_bounding_w_L_norm_unknown1}, $\mathbb{P}( \mathcal{E}_{\delta} \cap \bar{\mathcal{E}}_{0} \cap \mathcal{E}_2 ) \geq 1-4\delta$ 

By definition of $\pi_*$, $(\phi(\pi^*) - \phi(\pi) )^\top \w^* \geq 0$ for any arbitrary $\pi$. Therefore,

\begin{equation*}
    0 \leq (\phi(\pi^*) - \phi(\pi) )^\top \w^*
\end{equation*}

By Lemma~\ref{lemma::trajectory_score_concentration1}, with probability at least $1-\delta$ for all $\pi_1, \pi_2 \in \Pi$ simultaneously and all $t \in \mathbb{N}$

\begin{equation}
    (\phi(\pi_1) - \phi(\pi_2) )^\top \w^* \leq (\phi^{\widehat{\mathbb{P}}_t} (\pi_1) - \phi^{\widehat{\mathbb{P}}_t}(\pi_2) )^\top \w^* + \widehat{B}_t\left(\pi_1, 2SB, \frac{\delta}{2\mathcal{A}^{\mathcal{S}}}\right) + \widehat{B}_t\left(\pi_2, 2SB, \frac{\delta}{\mathcal{A}^{2\mathcal{S}}}\right), \label{equation::supporting_upper_bound_0}
\end{equation}

In particular this implies that with probability at least $1-\delta$ for $\pi^*$ and any $\pi$, 

\begin{equation}\label{equation::supporting_upper_bound_2}
0 \leq        (\phi(\pi^*) - \phi(\pi) )^\top \w^* \leq (\phi^{\widehat{\mathbb{P}}_t} (\pi^*) - \phi^{\widehat{\mathbb{P}}_t}(\pi) )^\top \w^* + \widehat{B}_t\left(\pi^*, 2SB, \frac{\delta}{2\mathcal{A}^{\mathcal{S}}}\right) + \widehat{B}_t\left(\pi, 2SB, \frac{\delta}{\mathcal{A}^{2\mathcal{S}}}\right),
\end{equation}

Let's bound the term $(\phi^{\widehat{\mathbb{P}}_t} (\pi^*) - \phi^{\widehat{\mathbb{P}}_t}(\pi) )^\top \w^* $

\begin{align}
    (\phi^{\widehat{\mathbb{P}}_t} (\pi^*) - \phi^{\widehat{\mathbb{P}}_t}(\pi) )^\top \w^*  &= (\phi^{\widehat{\mathbb{P}}_t} (\pi^*) - \phi^{\widehat{\mathbb{P}}_t}(\pi) )^\top \w_t^L + (\phi^{\widehat{\mathbb{P}}_t} (\pi^*) - \phi^{\widehat{\mathbb{P}}_t}(\pi) )^\top (\w^* - \w_t^L) \notag \\
    &\leq (\phi^{\widehat{\mathbb{P}}_t} (\pi^*) - \phi^{\widehat{\mathbb{P}}_t}(\pi) )^\top \w_t^L + \| \phi^{\widehat{\mathbb{P}}_t} (\pi^*) - \phi^{\widehat{\mathbb{P}}_t}(\pi) \|_{\widetilde{\mathbf{V}}^{-1}_t} \| \w^* - \w_t^L \|_{\widetilde{\mathbf{V}}_t} \label{equation::supporting_upper_bound_1}
\end{align}

Since $\bar{\mathcal{E}}_0$ holds, by Lemma~\ref{lemma::upper_bound_upper_bounding_w_L_norm_unknown1}

\begin{equation*}
        \| \w_t^L - \w_* \|_{\widetilde{\rmV}_t}  \leq \sqrt{2}\| \w_t^L - \w_* \|_{\bar{\rmV}_t} +  \sqrt{\sum_{\ell=1}^{t-1} 4\left(\widehat{B}_t\left(\pi, 2SB, \frac{\delta'}{8\ell^3\mathcal{A}^{\mathcal{S}} }\right)\right)^2 } + \frac{1}{t}.
\end{equation*}

Since $\mathcal{E}_\delta \cap \mathcal E_2$ is assumed to hold Corollary~\ref{corollary::V_bar_bound} implies that $ \| \w_t^L - \w_* \|_{\widetilde{\rmV}_t} \leq 4\kappa \beta_t(\delta) + \alpha_{d, T}(\delta)$ and therefore, 

\begin{equation*}
        \| \w_t^L - \w_* \|_{\widetilde{\rmV}_t}  \leq \sqrt{2} \left( 4\kappa \beta_t(\delta) + \alpha_{d, T}(\delta)  \right) +  \sqrt{\sum_{\ell=1}^{t-1} 4\left(\widehat{B}_t\left(\pi, 2SB, \frac{\delta'}{8\ell^3\mathcal{A}^{\mathcal{S}} }\right)\right)^2 } + \frac{1}{t}.
\end{equation*}

Since $\gamma_t = \sqrt{2} \left( 4\kappa \beta_t(\delta) + \alpha_{d, T}(\delta)  \right) +  \sqrt{\sum_{\ell=1}^{t-1} 4\left(\widehat{B}_t\left(\pi, 2SB, \frac{\delta'}{8\ell^3\mathcal{A}^{\mathcal{S}} }\right)\right)^2 } + \frac{1}{t}$, combining these results with Equations~\ref{equation::supporting_upper_bound_1} and~\ref{equation::supporting_upper_bound_2} yields,

\begin{equation*}
    0 \leq (\phi^{\widehat{\mathbb{P}}_t} (\pi^*) - \phi^{\widehat{\mathbb{P}}_t}(\pi) )^\top \w_t^L +\gamma_t  \| \phi^{\widehat{\mathbb{P}}_t} (\pi^*) - \phi^{\widehat{\mathbb{P}}_t}(\pi) \|_{\widetilde{\mathbf{V}}^{-1}_t}    +  \widehat{B}_t\left(\pi^*, 2SB, \frac{\delta}{2\mathcal{A}^{\mathcal{S}}}\right) + \widehat{B}_t\left(\pi, 2SB, \frac{\delta}{\mathcal{A}^{2\mathcal{S}}}\right),
\end{equation*}

Thus implying $\pi^* \in \Pi_t$. Taking a union bound between $\mathcal{E}_{\delta} \cap \bar{\mathcal{E}}_{0} \cap \mathcal{E}_2$ and the $1-\delta$ probability event from Equation~\ref{equation::supporting_upper_bound_1} yields the result. 

\end{proof}

\subsection{Proof of Lemma~\ref{lemma::regret_bound_unknown_support}}\label{section::proof_regret_bound_unknown_support}

Full version of Lemma~\ref{lemma::regret_bound_unknown_support},
\begin{restatable}{lemma}{lemmaregretboundunknownsupportappendix}\label{lemma::regret_bound_unknown_support_appendix}
If $\mathcal{E}_{\delta} \cap \mathcal{E}_2 \cap \bar{\mathcal{E}}_0 \cap \bar{\mathcal{E}}_{-1} \cap \bar{\mathcal{E}}_{2}$ the regret is bounded by,
\begin{align*}
    R_T &\leq 2\gamma_T \sqrt{2Td\log\left( 1 + \frac{TB}{d}\right) } + \\
    &\quad \sum_{t \in [T]} 4\widehat{B}_t(\pi_t^1, 4WB, \delta) + 4\widehat{B}_t(\pi_t^2, 4WB, \delta)
\end{align*}
\end{restatable}


\begin{proof}

We first condition on $ \mathcal{E}_\delta \cap \mathcal{E}_2 \cap \bar{\mathcal{E}}_0 \cap \bar{\mathcal{E}}_{-1}$. Let's start by showing the following bound on the instantaneous regret,
\begin{align*}
     2 r_t &\leq 2\gamma_t  \| \phi^{\widehat{\mathbb{P}}_t}(\pi_t^1) - \phi^{\widehat{\mathbb{P}}_t}(\pi_t^2) \|_{\widetilde \V_t^{-1}}  + 4\widehat{B}_t(\pi_t^1, 4SB, \delta) + 4\widehat{B}_t(\pi_t^2, 4SB, \delta) 
\end{align*}
Since we are conditioning on $\bar{\mathcal{E}}_2$, by Lemma~\ref{lemma::first_unknown_model_lemma} follows that for all $t$,
\begin{align*}
 2r_t  &\leq  (\phi^{\widehat{\mathbb{P}}_t}(\pi^*) - \phi^{\widehat{\mathbb{P}}_t}(\pi_t^1))^\top \w^* + (\phi^{\widehat{\mathbb{P}}_t}(\pi^*) - \phi^{\widehat{\mathbb{P}}_t}(\pi_t^2))^\top \w^* +\widehat{B}_t(\pi^*, 4SB, \delta) +  \widehat{B}_t(\pi_t^1, 2SB, \delta)  + \widehat{B}_t(\pi_t^2, 2SB, \delta).
\end{align*}
Let's focus on bounding the term $(\phi^{\widehat{\mathbb{P}}_t}(\pi^*) - \phi^{\widehat{\mathbb{P}}_t}(\pi_t^1))^\top \w^* + (\phi^{\widehat{\mathbb{P}}_t}(\pi^*) - \phi^{\widehat{\mathbb{P}}_t}(\pi_t^2))^\top \w^*$.
\begin{align}
    (\phi^{\widehat{\mathbb{P}}_t}(\pi^*) - \phi^{\widehat{\mathbb{P}}_t}(\pi_t^1))^\top \w^* + (\phi^{\widehat{\mathbb{P}}_t}(\pi^*)& - \phi^{\widehat{\mathbb{P}}_t}(\pi_t^2))^\top \w^*  = (\phi^{\widehat{\mathbb{P}}_t}(\pi^*) - \phi^{\widehat{\mathbb{P}}_t}(\pi_t^1))^\top \w^L_t + (\phi^{\widehat{\mathbb{P}}_t}(\pi^*) - \phi^{\widehat{\mathbb{P}}_t}(\pi_t^1))^\top (\w^* - \w_t^L) \notag \\
    &\quad + (\phi^{\widehat{\mathbb{P}}_t}(\pi^*) - \phi^{\widehat{\mathbb{P}}_t}(\pi_t^2))^\top \w^L_t + (\phi^{\widehat{\mathbb{P}}_t}(\pi^*) - \phi(\pi_t^2))^\top (\w^* - \w^L_t) \notag   \\
& \leq (\phi^{\widehat{\mathbb{P}}_t}(\pi^*) - \phi^{\widehat{\mathbb{P}}_t}(\pi_t^1))^\top \w^L_t + (\phi^{\widehat{\mathbb{P}}_t}(\pi^*) - \phi^{\widehat{\mathbb{P}}_t}(\pi_t^2))^\top \w^L_t  \notag\\
& \quad + \|\w^* - \w^L_t \|_{\widetilde \V_t} \cdot \| \phi^{\widehat{\mathbb{P}}_t}(\pi^*) - \phi^{\widehat{\mathbb{P}}_t}(\pi_t^1) \|_{\widetilde \V_t^{-1}} + \notag \\
&\quad \|\w^* - \w^L_t \|_{\widetilde \V_t} \cdot \| \phi^{\widehat{\mathbb{P}}_t}(\pi^*) - \phi^{\widehat{\mathbb{P}}_t}(\pi_t^2) \|_{\widetilde \V_t^{-1}} \label{equation::support_equation_lemma_regret_bound_unknown_support}
\end{align}
Since $\mathcal{E}_\delta \cap \mathcal{E}_2 \cap \bar{\mathcal{E}}_0$ holds, the last two terms in the sum above can be bounded using Lemma~\ref{lemma::upper_bound_upper_bounding_w_L_norm_unknown1} and Corollary~\ref{corollary::V_bar_bound} by 
\begin{align}
    \|\w^* - \w^L_t \|_{\widetilde \V_t} \cdot \| \phi^{\widehat{\mathbb{P}}_t}(\pi^*) - \phi^{\widehat{\mathbb{P}}_t}(\pi_t^1) \|_{\widetilde \V_t^{-1}} + \|\w^* - \w^L_t \|_{\widetilde \V_t} \cdot \| \phi^{\widehat{\mathbb{P}}_t}(\pi^*) - \phi^{\widehat{\mathbb{P}}_t}(\pi_t^2) \|_{\widetilde \V_t^{-1}} \notag \\
    \leq  \gamma_t  \left(  \| \phi^{\widehat{\mathbb{P}}_t}(\pi^*) - \phi^{\widehat{\mathbb{P}}_t}(\pi_t^1) \|_{\widetilde \V_t^{-1}}  +  \| \phi^{\widehat{\mathbb{P}}_t}(\pi^*) - \phi^{\widehat{\mathbb{P}}_t}(\pi_t^2) \|_{\widetilde \V_t^{-1}}    \right) \notag 
\end{align}
Where $$\gamma_t = \sqrt{2} \left( 4\kappa \beta_t(\delta) + \alpha_{d, T}(\delta)  \right) +  \sqrt{\sum_{\ell=1}^{t-1} 4\left(\widehat{B}_t\left(\pi_\ell^1, 2SB, \frac{\delta'}{8\ell^3\mathcal{A}^{\mathcal{S}} }\right)\right)^2 + 4\left(\widehat{B}_t\left(\pi_\ell^2, 2SB, \frac{\delta'}{8\ell^3\mathcal{A}^{\mathcal{S}} }\right)\right)^2  } + \frac{1}{t}.$$

The first two terms on the right hand side of inequality~\ref{equation::support_equation_lemma_regret_bound_unknown_support} leverage the optimistic bonus, using the fact that $\pi_t^1, \pi_t^2 \in \Pi_t$ and therefore,
\begin{align*}
    (\phi^{\widehat{\mathbb{P}}_t}(\pi^*) - \phi^{\widehat{\mathbb{P}}_t}(\pi_t^1))^\top \w^L_t + (\phi^{\widehat{\mathbb{P}}_t}(\pi^*) - \phi^{\widehat{\mathbb{P}}_t}(\pi_t^2))^\top \w^L_t  +  \widehat{B}_t(\pi^*, 4SB, \delta) +  \widehat{B}_t(\pi_t^1, 2SB, \delta)  +  \widehat{B}_t(\pi_t^2, 2SB, \delta) \\
    \leq  \gamma_t\left(    \| \phi^{\widehat{\mathbb{P}}_t}(\pi^*) - \phi^{\widehat{\mathbb{P}}_t}(\pi_t^1) \|_{\widetilde \V_t^{-1}}  +  \| \phi^{\widehat{\mathbb{P}}_t}(\pi^*) - \phi^{\widehat{\mathbb{P}}_t}(\pi_t^2) \|_{\widetilde \V_t^{-1}}      \right)    +   3\widehat{B}_t(\pi^*, 4SB, \delta) + \widehat{B}_t(\pi_t^1, 4SB, \delta) + \widehat{B}_t(\pi_t^2, 4SB, \delta) \\
 \leq   \gamma_t\left(    \| \phi^{\widehat{\mathbb{P}}_t}(\pi^*) - \phi^{\widehat{\mathbb{P}}_t}(\pi_t^1) \|_{\widetilde \V_t^{-1}}  +  \| \phi^{\widehat{\mathbb{P}}_t}(\pi^*) - \phi^{\widehat{\mathbb{P}}_t}(\pi_t^2) \|_{\widetilde \V_t^{-1}}      \right)    +   4\widehat{B}_t(\pi^*, 4SB, \delta) + 2\widehat{B}_t(\pi_t^1, 4SB, \delta) + 2\widehat{B}_t(\pi_t^2, 4SB, \delta)
\end{align*}
Putting these together we can conclude that,
\begin{align*}
   2 r_t &\leq 2\gamma_t\left(    \| \phi^{\widehat{\mathbb{P}}_t}(\pi^*) - \phi^{\widehat{\mathbb{P}}_t}(\pi_t^1) \|_{\widetilde \V_t^{-1}}  +  \| \phi^{\widehat{\mathbb{P}}_t}(\pi^*) - \phi^{\widehat{\mathbb{P}}_t}(\pi_t^2) \|_{\widetilde \V_t^{-1}}      \right)    +  \\
   &\quad 4\widehat{B}_t(\pi^*, 4SB, \delta) + 2\widehat{B}_t(\pi_t^1, 4SB, \delta) + 2\widehat{B}_t(\pi_t^2, 4SB, \delta)
\end{align*}
Recall that whenever $\bar{\mathcal{E}}_{-1}$ holds, $\pi^* \in \Pi_t$ and that as a result of how $\pi_t^1,\pi_t^2$ are chosen (see Algorithm~\ref{alg:urrl})
\begin{align*}
     2 r_t &\leq 2\gamma_t  \| \phi^{\widehat{\mathbb{P}}_t}(\pi_t^1) - \phi^{\widehat{\mathbb{P}}_t}(\pi_t^2) \|_{\widetilde \V_t^{-1}}  + 4\widehat{B}_t(\pi_t^1, 4SB, \delta) + 4\widehat{B}_t(\pi_t^2, 4SB, \delta) 
\end{align*}
The regret is therefore upper bounded by,
\begin{align*}
    R_T &= \sum_{t \in [T]} 2r_t\\
    &\leq \sum_{t \in [T]} 2\gamma_t  \| \phi^{\widehat{\mathbb{P}}_t}(\pi_t^1) - \phi^{\widehat{\mathbb{P}}_t}(\pi_t^2) \|_{\widetilde \V_t^{-1}}  + 4\widehat{B}_t(\pi_t^1, 4SB, \delta) + 4\widehat{B}_t(\pi_t^2, 4SB, \delta) \\
    & \leq 2\gamma_T \sqrt{ T \sum_{t \in [T]}  \| \phi^{\widehat{\mathbb{P}}_t}(\pi_t^1) - \phi^{\widehat{\mathbb{P}}_t}(\pi_t^2) \|^2_{\widetilde \V_t^{-1}}    } + 4\widehat{B}_t(\pi_t^1, 4SB, \delta) + 4\widehat{B}_t(\pi_t^2, 4SB, \delta) \\
    &\leq 2\gamma_T \sqrt{2Td\log\left( 1 + \frac{TB}{d}\right) } + \sum_{t \in [T]} 4\widehat{B}_t(\pi_t^1, 4SB, \delta) + 4\widehat{B}_t(\pi_t^2, 4SB, \delta)
\end{align*}
Where the last inequality follows from Lemma~\ref{lem::determinant}. 

\end{proof}

\section{Miscelaneous Technical Lemmas}

We will make use of the following Lemmas

\begin{lemma}[Hoeffding Inequality]\label{lemma::matingale_concentration_anytime}
Let $\{x_t\}_{t=1}^\infty$ be a martingale difference sequence with $| x_t | \leq \zeta$ and let $\delta \in (0,1]$. Then with probability $1-\delta$ for all $T \in \mathbb{N}$
\begin{equation*}
    \sum_{t=1}^T x_t \leq 2\zeta \sqrt{T \ln \left(\frac{6\ln T}{\delta}\right) }.
\end{equation*}
\end{lemma}
\begin{proof}Observe that $\frac{\left|x_t\right|}{\zeta}\le 1$. By invoking a time-uniform Hoeffding-style concentration inequality \citep[][Equation~(11)]{howard2018time} we find that
\begin{align*}
    \Pr\left[\forall \;  t \in \N\;:\; \sum_{t=1}^{T} \frac{x_t}{\zeta} \le 1.7\sqrt{T \left(\log\log(T)+0.72\log\left(\frac{5.2}{\delta}\right)\right)} \right]\ge 1-\delta.
\end{align*}
Rounding up the constants for the sake of simplicity we get
\begin{align*}
      \Pr\left[\forall \; t \in \N\;:\; \sum_{t=1}^{T} x_t \le 2\zeta\sqrt{T \left(\log\left(\frac{6\log(T)}{\delta}\right)\right)} \right]\ge 1-\delta,  
\end{align*}
which establishes our claim.
\end{proof}

\begin{lemma}[Uniform empirical Bernstein bound]
\label{lem:uniform_emp_bernstein}
In the terminology of \citet{howard2018uniform}, let $S_t = \sum_{i=1}^t Y_i$ be a sub-$\psi_P$ process with parameter $c > 0$ and variance process $W_t$. Then with probability at least $1 - \delta$ for all $t \in \mathbb{N}$
\begin{align*}
    S_t &\leq  1.44 \sqrt{\max(W_t , m) \left( 1.4 \ln \ln \left(2 \left(\max\left(\frac{W_t}{m} , 1 \right)\right)\right) + \ln \frac{5.2}{\delta}\right)}\\
   & \qquad + 0.41 c  \left( 1.4 \ln \ln \left(2 \left(\max\left(\frac{W_t}{m} , 1\right)\right)\right) + \ln \frac{5.2}{\delta}\right)
\end{align*}
where $m > 0$ is arbitrary but fixed.
\end{lemma}



\end{document}